\documentclass{article}

    \PassOptionsToPackage{numbers, compress}{natbib}

 \usepackage[preprint]{neurips_2026}

\usepackage[utf8]{inputenc} %
\usepackage[T1]{fontenc}    %
\usepackage{hyperref}
            
\usepackage{url}            %
\usepackage{booktabs}       %
\usepackage{amsfonts}       %
\usepackage{nicefrac}       %
\usepackage{microtype}      %

\usepackage[table, svgnames, dvipsnames]{xcolor}
\hypersetup{
    colorlinks=true,
    linkcolor=black,
    citecolor=[rgb]{0.0,0.2,0.6},
    urlcolor=[rgb]{0.0,0.2,0.6}
}
\usepackage{amsthm}
\usepackage{algorithm}
\usepackage{algpseudocode}
\usepackage{amssymb}
\usepackage{bbm}
\usepackage{amsmath}
\usepackage{mathtools}
\usepackage{xspace}
\usepackage{tikz}
\usetikzlibrary{arrows.meta, positioning, automata, matrix}
\newcommand{\fcell}[2]{\fill[black!85] (#1, #2) rectangle ++(0.55, 0.55);}
\usepackage{subcaption}
\usepackage{enumitem}
\usepackage{tabularx}
\usepackage{makecell}
\usepackage{longtable}
\usepackage{wrapfig}
\definecolor{spanred}{RGB}{200,30,30}
\usepackage{array, booktabs, multirow}
\newcolumntype{L}{>{\raggedright\arraybackslash}X}
\usepackage{colortbl}
\definecolor{vrulecolor}{gray}{0.80}

\usepackage{amsmath,amsfonts,bm}

\def\eqref#1{equation~\ref{#1}}

\def\1{\bm{1}}

\DeclareMathAlphabet{\mathsfit}{\encodingdefault}{\sfdefault}{m}{sl}
\SetMathAlphabet{\mathsfit}{bold}{\encodingdefault}{\sfdefault}{bx}{n}

\newcommand\given{\mid}
\newcommand{\mask}{\ensuremath{\texttt{[MASK]}}\xspace}

\newcommand{\ind}[1]{\mathbbm{1}[#1]}

\definecolor{cleft}{RGB}{160, 40, 40}     %
\definecolor{cmid}{RGB}{30, 100, 50}      %
\definecolor{cright}{RGB}{30, 60, 140}    %

\newcommand{\zl}{\ensuremath{\textcolor{cleft}{z_\ell}}}
\newcommand{\zm}{\ensuremath{\textcolor{cmid}{z_m}}}

\newcommand{\zr}{\ensuremath{\textcolor{cright}{z_r}}}

\newcommand{\xlr}{\ensuremath{x_{\ell:r-1}}}
\newcommand{\xlm}{\ensuremath{x_{\ell:m-1}}}
\newcommand{\xmr}{\ensuremath{x_{m:r-1}}}

\newcommand\pfa{p_\mathcal{M}}
\newcommand\dlm{dLLM\xspace}
\newcommand\dlms{dLLMs\xspace}
\newcommand\llada{LLaDA\xspace}
\newcommand\dream{Dream\xspace}

\newcommand{\src}{\mathrm{src}}
\newcommand{\dst}{\mathrm{dst}}
\newcommand{\lbl}{\mathrm{label}}

\theoremstyle{plain}
\newtheorem{theorem}{Theorem}[section]
\newtheorem{proposition}[theorem]{Proposition}

\theoremstyle{definition}
\newtheorem{definition}[theorem]{Definition}

\theoremstyle{remark}

\title{Constrained Decoding for Diffusion Language Models via Efficient Inference over Finite Automata}

\author{%
  Meihua Dang \\
  Stanford University \\
  \texttt{mhdang@cs.stanford.edu}
  \And
  Stefano Ermon \\
  Stanford University \\
  \texttt{ermon@cs.stanford.edu} \\
}

\begin{document}

\maketitle

\begin{abstract}

Constrained decoding is essential for serving LLMs, ensuring that generated outputs follow specific structures such as JSON schema-formatted function calls. Existing systems are designed for autoregressive models and assume left-to-right generation, masking out invalid next tokens at each step. 
Diffusion language models, however, break this assumption: they sample multiple positions simultaneously from a fully-factorized mean-field distribution at each denoising step. 
In this paper, we present an exact and tractable algorithm for sampling from the constrained mean-field posterior under any constraint expressible as a finite automaton. Viewing finite automata as graphical models, we obtain tractable representations of the constrained distribution that enable efficient inference.
The approach guarantees constraint satisfaction by construction, supports both greedy and sampling-based decoding, and is compatible with parallel and block-wise decoding under arbitrary remasking schedules.
Applying depth-reduction techniques from arithmetic circuit theory, we further reduce sampling depth from linear to logarithmic in the sequence length.
Empirical evaluations on Dream-7B and LLaDA-8B show substantial accuracy gains across various tasks including function calling (xLAM, BFCL), planning (Sudoku, Countdown), text-to-SQL (Spider), and math reasoning (GSM-Symbolic), with little inference overhead relative to unconstrained decoding. For example, on BFCL-Live, our approach improves Dream-7B's greedy decoding accuracy from 63.9\% to 71.5\%, and stochastic sampling accuracy from 22.3\% to 69.0\%, where the unconstrained baseline collapses, with under 5\% wall-clock overhead.

\end{abstract}
\section{Introduction}
\label{sec:introduction}

A wide range of LLM applications require outputs to strictly adhere to structural constraints. Examples include generating function calls conforming to a given JSON schema~\cite{patil2024gorrilla,patil2025the} and producing syntactically correct SQL queries~\cite{lei2025spider}. Autoregressive LLMs enforce such constraints at inference time via \emph{constrained decoding}, where invalid next tokens violating the constraint are masked out at each generation step~\citep{willard2023efficient}. This approach is now a standard component of LLM inference frameworks such as XGrammar~\citep{dong2025xgrammar}, Guidance~\cite{guidance_repo} and SGLang~\cite{zheng2024sglang}.

Recent advances in discrete diffusion language models (\dlms{}) have established them as competitive alternatives to autoregressive LLMs~\citep{austin2021d3pm,lou2023discrete,sahoo2024simple}. Open-weight models such as \llada-8B~\citep{nie2025large} and \dream-7B~\citep{ye2025dream} have demonstrated performance comparable to similarly sized autoregressive models on code and math reasoning. Closed-source models including Mercury~\cite{inception25mercury}, Gemini Diffusion~\cite{geminidiffusion2025}, and Seed Diffusion~\cite{song2025seed} further demonstrate excellent inference efficiency. 

Existing constrained generation frameworks~\citep{dong2025xgrammar,guidance_repo,zheng2024sglang} are specifically designed for left-to-right autoregressive generation and adapting them to \dlms poses unique challenges. Unlike autoregressive LLMs, \dlms sample multiple tokens simultaneously from a fully-factorized mean-field distribution at each step. Because constraints operate globally, masking tokens independently at each position is not enough: locally valid tokens can still combine into invalid sequences.

To bridge this gap, we introduce an exact and tractable inference algorithm that, at each diffusion step, samples from the \emph{constrained mean-field posterior}: $p_\theta(x^0\!\mid\!x^t, \mathcal{C})\!\propto\!p_\theta(x^0\!\mid\!x^t) \cdot \ind{x^0\!\in\!\mathcal{C}}$ where the constraint $\mathcal{C}$ is expressed as a finite automaton, $p_\theta(x^0 \mid x^t)$ is the model’s fully factorized mean-field prediction, and the indicator enforces the constraint globally on the complete sequence~$x^0$.
Viewing the automaton as a graphical model encoding constraint $\mathcal{C}$ as its support, we are able to obtain a tractable representation for the constrained mean-field prediction $p_\theta(x^0\!\mid\!x^t, \mathcal{C})$, which enables efficient sampling. 
Our approach guarantees constraint satisfaction by construction, supports both greedy decoding and stochastic sampling, and applies to parallel and block-wise decoding under arbitrary remasking schedules. It also computes per-token marginals under the constrained distribution, which we use as a confidence signal for remasking.

Naive sampling from the chain-structured graphical model induced by a finite automaton requires ${\mathcal{O}}(L)$ \emph{sequential} forward/backward steps for a sequence of length $L$. Inspired by depth-reduction techniques for arithmetic circuits~\citep{valiant1981fast}, we restructure this linear computation graph into a binary tree of prefix–suffix state pairs, reducing inference steps from ${\mathcal{O}}(L)$ to ${\mathcal{O}}(\log L)$ while preserving the same sampling distribution. This yields substantially more parallelism on modern GPUs.

Empirically, we evaluate our method on \llada-8B and \dream-7B comparing against the unconstrained base models. We consider diverse constraints and tasks, including JSON-schema constraints for function calling (xLAM, BFCL), format constraints for planning (Sudoku, Countdown), SQL syntax and schema constraints (Spider), and symbolic expression constraints for math reasoning (GSM-Symbolic). Our method guarantees constraint satisfaction and substantially improves accuracy. For example, on BFCL-Live, our approach improves \dream's greedy decoding accuracy from 63.9\% to 71.5\%, and remains robust under stochastic sampling: while the unconstrained \dream collapses to 22.3\%, our method achieves 69.0\%, with under 5\% wall-clock overhead.

\section{Related Work}
\paragraph{Constrained decoding for autoregressive LLMs.}
There is an extensive line of work on locally constrained decoding~\cite{shin2021constrained, scholak2021picard, poesia2022synchromesh, ugare2025syncode, park2025flexible}, with a primary focus on enforcing hard constraints efficiently during autoregressive generation. These approaches typically operate by masking out invalid next tokens based on the current prefix, thereby ensuring local constraint satisfaction at each decoding step. Existing frameworks leverage different constraint representations and develop different parsing algorithms correspondingly: for example, Outlines~\cite{willard2023efficient} uses regular expressions and their corresponding finite automata representations, while more recent systems such as XGrammar~\citep{dong2025xgrammar}  and Guidance~\cite{guidance_repo} rely on context-free grammars. 

\paragraph{Constrained decoding for \dlms.}
DINGO~\citep{suresh2026dingo} is one pioneering work exploring inference algorithms over finite automata to support constrained decoding for \dlms. In particular, DINGO leverages a dynamic programming algorithm to perform MAP inference on deterministic finite automata with their inputs weighted by the mean-field predictions from \dlms, supporting block-wise greedy decoding from \dlms. Our approach can be viewed as a strict generalization of DINGO, supporting greedy decoding and sampling under arbitrary masking strategies constrained by (possibly nondeterministic) finite automata~(Section~\ref{sec:method}).
\citet{mundler2026constrained} and \citet{zhang2026lookahead} target richer constraints expressed as context-free grammars~(CFGs). Both adopt a per-step rejection-sampling scheme that draws candidate tokens and then verifies whether the intermediate output remains acceptable; \citet{zhang2026lookahead} additionally introduce a lookahead step that uses the mean-field distributions to reliably verify proposed tokens. These methods do not guarantee constraint satisfaction \emph{within a finite sampling budget}, as acceptance rate can be extremely low in practice. 

\section{Background}
\label{sec:background}

\subsection{Diffusion Language Models}
\label{sec:dllm}
Masked \dlms~\cite{austin2021d3pm,sahoo2024simple,shi2024simplified} learn to reverse a forward corruption process $q$ that stochastically converts a clean token sequence $x^0 = (x^0_1, \ldots, x^0_L)$ of length $L$ into a noisy sequence $x^t$ over time $t \in [0, 1]$, by masking each token with probability $t$:
\begin{equation}
    q_{t \given 0}(x^t_i \given x^0_i) =
    \begin{cases}
        t,     & x_i^t = \mask \\
        1 - t, & x_i^t = x_i^0
    \end{cases}
    \qquad
    q_{t \given 0}(x^t \given x^0) = \prod_{i} q_{t \given 0}(x^t_i \given x^0_i).
    \label{eq:dlm-forward}
\end{equation}
\dlms parameterize the reverse process via a mean-field distribution $p_\theta(x^0 \!\mid\! x^t)$ over clean tokens given noisy input, trained by maximizing an evidence lower bound on the data log-likelihood~\cite{sahoo2024simple,ou2025your}. 

Sampling from \dlms begins with a fully masked sequence and proceeds by iteratively unmasking tokens. A denoising step from time $t$ to $s < t$ first draws a full clean sequence $\hat{x}^0$ from the mean-field prediction $p_\theta(x^0 \given x^t)$, then commits a subset $\mathcal{U} \subseteq \{i : x_i^t = \mask\}$ of newly predicted positions and re-masks the rest to obtain $x^s$:
\begin{equation}
     \hat{x}^0\sim p_\theta(x^0 \given x^t) = \prod_{i=1}^L p_\theta(x_i^0 \given x^t),
    \qquad
    x_i^{s} = \begin{cases} \hat{x}^0_i & i \in \mathcal{U} \\ x_i^t & \text{otherwise.} \end{cases}
    \label{eq:dlm-sampling}
\end{equation}
Iterating this denoising process from $x^T$ down to $x^0$ defines a marginal distribution $p_{\mathrm{\dlm}}(x^0)$ over clean sequences, which we refer to as the \dlm's \emph{generative distribution}.
While \dlms are trained on uniformly random masks, in practice decoding heuristics are used at inference time to choose which positions to unmask. For example, \dream selects positions with the lowest predictive entropy at each step~\cite{ye2025dream}, while \llada selects positions with the highest predictive probability~\cite{nie2025large}.

\subsection{Finite Automata}
\label{sec:back-nfa}

In this work, we consider constraints expressible as \emph{finite automata}~\citep{hopcroft2001introduction}. 

\begin{definition}
A (nondeterministic) finite automaton (FA) is a tuple $\mathcal{M} = (\mathcal{S}, \mathcal{V}, \delta, s_0, \mathcal{F})$ with states $\mathcal{S}$, alphabet $\mathcal{V}$ (often corresponding to a \dlm vocabulary), transition function $\delta : \mathcal{S} \times \mathcal{V} \to 2^{\mathcal{S}}$, initial state $s_0$, and accept states $\mathcal{F} \subseteq \mathcal{S}$. A string $v_{1:L}$ is \emph{accepted} if there exists a sequence of states $s_0, s_1, \dots, s_L$ with $s_i \in \delta(s_{i-1}, v_i)$ and $s_L \in \mathcal{F}$. A finite automaton is \emph{deterministic} if $|\delta(s, v)| = 1$ for all $s$ and $v$.
\end{definition}
A finite automaton encodes a constraint $\mathcal{C} \subseteq \mathcal{V}^L$ as the set of length-$L$ accepted strings. Equivalently, as illustrated in Figure~\ref{fig:fa_gm}, we can also view $\mathcal{M}$ as a directed graph with \emph{edge set}
\[
\mathcal{E}
:=
\{(s, s') \in \mathcal{S} \times \mathcal{S}
\mid \exists v \text{ s.t. } s'\!\in \!\delta(s, v)\}.
\]
For each edge $e\!=\!(s, s')$, we write $\src(e)\!:=\!s$, $\dst(e)\!:=\!s'$ and $\lbl(e) := \{v \mid s'\!\in\!\delta(s, v)\}$. Under this view, a string $v_{1:L}$ is accepted iff there exists a path $e_{1:L}$ such that $\src(e_1)\!=\!s_0$, $\dst(e_i)\!=\!\src(e_{i+1})$ for all $i<L$, $\dst(e_L)\in\mathcal{F}$, and $v_i \in \lbl(e_i)$ for all $i$.

\paragraph{Finite automata as graphical models.}

\begin{figure*}[t]
    \centering
    \resizebox{\textwidth}{!}{\providecommand{\fcell}[2]{\fill[black!85] (#1, #2) rectangle ++(0.55, 0.55);}

\begin{tikzpicture}[
    every node/.style={font=\normalsize, inner sep=2.5pt},
    arr/.style={-{Stealth[length=2.5mm]}, thick},
    state/.style={circle, draw, thick, minimum size=8mm, inner sep=0pt},
    accept/.style={state, double, double distance=1.5pt},
    cellgrid/.style={step=0.55, very thin, gray!60},
]

\begin{scope}[xshift=0cm, yshift=0cm]
\node[state]  (s0) at (0,   1.4) {$s_0$};
\node[state]  (s1) at (2.0, 1.4) {$s_1$};
\node[accept] (s2) at (4.0, 1.4) {$s_2$};
\draw[arr] (s0) -- node[above] {$a\;{\scriptstyle(e_0)}$} (s1);
\draw[arr] (s1) edge[loop above, looseness=6] node {$b,c\;{\scriptstyle(e_1)}$} (s1);
\draw[arr] (s1) -- node[above] {$d\;{\scriptstyle(e_2)}$} (s2);
\draw[arr] (s2) edge[loop above, looseness=6] node {$e\;{\scriptstyle(e_3)}$} (s2);
\node at (2.0, -0.6) {(a) DFA $\mathcal{M}$};
\end{scope}

\begin{scope}[xshift=5.2cm, yshift=0cm]
\node          (z1)   at (0,   1.6) {$z_1$};
\node          (z2)   at (1.0, 1.6) {$z_2$};
\node          (zd)   at (2.0, 1.6) {$\cdots$};
\node          (zL)   at (3.0, 1.6) {$z_L$};
\node[draw, fill=black!12, thick, rounded corners=1pt, minimum size=6mm,
      inner sep=1pt] (zLp1) at (4, 1.6) {\scriptsize $z_{L\!+\!1}$};
\node (x1) at (0,   0.7) {$x_1$};
\node (x2) at (1.0, 0.7) {$x_2$};
\node (xL) at (3.0, 0.7) {$x_L$};
\draw[arr] (z1) -- (z2);
\draw[arr] (z2) -- (zd);
\draw[arr] (zd) -- (zL);
\draw[arr] (zL) -- (zLp1);
\draw[arr] (z1) -- (x1);
\draw[arr] (z2) -- (x2);
\draw[arr] (zL) -- (xL);
\node at (2.15, -0.6) {(b) Graphical Model};
\end{scope}

\begin{scope}[xshift=10.8cm, yshift=0cm]
\node at (0.275, 2.50) {$e_0$};
\node at (0.825, 2.50) {$e_1$};
\node at (1.375, 2.50) {$e_2$};
\node at (1.925, 2.50) {$e_3$};
\node at (-0.40, 1.925) {$e_0$};
\node at (-0.40, 1.375) {$e_1$};
\node at (-0.40, 0.825) {$e_2$};
\node at (-0.40, 0.275) {$e_3$};
\fcell{0.55}{1.65}    %
\fcell{1.10}{1.65}    %
\fcell{0.55}{1.10}    %
\fcell{1.10}{1.10}    %
\fcell{1.65}{0.55}    %
\fcell{1.65}{0.0}     %
\draw[cellgrid] (0,0) grid[step=0.55] (2.20, 2.20);
\draw[thick] (0,0) rectangle (2.20, 2.20);
\node at (1.10, -0.6) {(c) Transition };
\end{scope}

\begin{scope}[xshift=14.2cm, yshift=0cm]
\node at (0.275, 2.50) {$a$};
\node at (0.825, 2.50) {$b$};
\node at (1.375, 2.50) {$c$};
\node at (1.925, 2.50) {$d$};
\node at (2.475, 2.50) {$e$};
\node at (-0.40, 1.925) {$e_0$};
\node at (-0.40, 1.375) {$e_1$};
\node at (-0.40, 0.825) {$e_2$};
\node at (-0.40, 0.275) {$e_3$};
\fcell{0.0}{1.65}     %
\fcell{0.55}{1.10}    %
\fcell{1.10}{1.10}    %
\fcell{1.65}{0.55}    %
\fcell{2.20}{0.0}     %
\draw[cellgrid] (0,0) grid[step=0.55] (2.75, 2.20);
\draw[thick] (0,0) rectangle (2.75, 2.20);
\node at (1.375, -0.6) {(d) Emission};
\end{scope}

\end{tikzpicture}}
\caption{\textbf{(a)} A finite automaton $\mathcal{M}$ accepting $a(b|c)^* d e^*$. \textbf{(b)} The induced chain-structured graphical model over length-$L$ sequences, with edge-valued latents $z_{1:L}$ emitting observed tokens $x_{1:L}$. \textbf{(c)} The transition factor $p(z_t \!\mid \! z_{t-1})$, nonzero when consecutive edges chain end-to-end ($\dst(z_{t-1}) = \mathrm{src}(z_t)$). \textbf{(d)} The emission factor $p(x_t \!\mid \! z_t)$, nonzero when $x_t$ is in the label set of edge $z_t$.}
\label{fig:fa_gm}
\end{figure*}
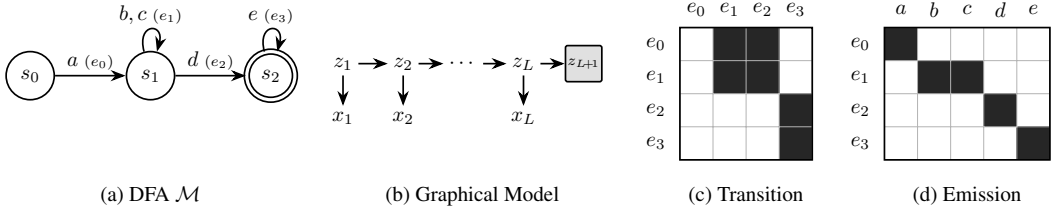

Following~\citet{dang2026mitigating}, we define $\pfa$ as a chain-structured graphical model over sequences of $L$ tokens $\pfa(x_{1:L})$ such that $\pfa(x_{1:L})\!>\!0$ if and only if $x_{1:L}\!\in\!\mathcal{C}$, that is, the constraint $\mathcal{C}$ is encoded as its support. $\pfa$ can also be viewed as a hidden Markov model~\citep{rabiner1986introduction}, with latent variables $z_i\!\in\!\mathcal{E}$ taking values from the edge set of the automaton. See Figure~\ref{fig:fa_gm} for a visualization. 

The latent sequence $z_{1:L}$ encodes a random path through the automaton: $z_t$ is the edge taken at step $i$ with token $x_i$ emitted from its label $\lbl(e)$. Consecutive edges must chain end-to-end, i.e. $\dst(z_i)\!=\!\src(z_{i+1})$, as required by the transition matrix, starting at the initial state $s_0$ and ending in an accepting state in $\mathcal{F}$. Abusing notation, we write conditional probabilities as\footnote{These are unnormalized factors rather than conditional probabilities; we treat them as such and implicitly normalize globally over length-$L$ sequences during sampling.}
\[
\left\{
\begin{aligned}
\pfa(z_1 = e)
&:= \ind{\src(e) = s_0}, \\[0.4em]
\pfa(z_i = e' \given z_{i-1} = e)
&:= \ind{\dst(e) = \src(e')}, \\[0.4em]
\pfa(x_i = v \given z_i = e)
&:= \ind{v \in \lbl(e)}.
\end{aligned}
\right.
\]
Marginalizing over the latent paths yields
\begin{equation}
\pfa(x_{1:L})
\propto \sum_{z_{1:L}}
\pfa(z_1)
\prod_{i=2}^{L} \pfa(z_i \given z_{i-1})
\prod_{i=1}^{L} \pfa(x_i \given z_i)
\;\ind{\dst(z_L) \in \mathcal{F}}.
\end{equation}
The value of $\pfa(x_{1:L})$ is proportional to the number of accepting paths that generate $x_{1:L}$; if the finite automaton is \emph{deterministic}, then each accepted sequence has exactly one accepting path, thus $\pfa$ becomes uniform over $\mathcal{C}$. The chain structure allows $\pfa$ and its marginals to be computed exactly via forward–backward message passing~\citep{pearl1982reverend}.

\section{Method}
\label{sec:method}

\subsection{Problem Statement}
\label{sec:method-problem}
Our goal is to sample from a \dlm's generative distribution $p_{\mathrm{\dlm}}$ restricted to a constraint $\mathcal{C}$, i.e., $p_{\mathrm{\dlm}}(x^0 \!\mid\! x^0 \in \mathcal{C})$. However, sampling exactly from this distribution is intractable, and one naive resort is rejection sampling against the full reverse process, which is infeasible in practice due to extremely low acceptance rates.
For autoregressive LLMs, a common practice is to mask out next tokens that would violate the constraint as generation proceeds left-to-right~\citep{willard2023efficient}. Following an analogous strategy, we enforce that samples drawn from \emph{each diffusion step} must satisfy the constraint. Recall from Equation~(\ref{eq:dlm-sampling}) that a standard diffusion step samples $\hat{x}^0$ from the model's mean-field distribution.
We instead sample from this distribution \emph{restricted to} $\mathcal{C}$, ensuring $\hat{x}^0 \in \mathcal{C}$ at every step, so the final output is guaranteed to satisfy the constraint by construction. Concretely, at each denoising step we draw
\begin{equation}
    \label{eq:condition}
    \hat{x}^0 \;\sim\; p_\theta(x^0 \given x^t,\, \mathcal{C})
    \;\propto\;
    \left(\prod_{i=1}^{L} p_\theta(x_i^0 \given x^t)\right) \cdot \ind{x^0 \in \mathcal{C}}.
\end{equation}
\paragraph{Why this is challenging.} At each diffusion step, the mean-field prediction $\prod_i p_\theta(x_i^0 \given x^t)$ is fully factorized, i.e. each position is sampled independently, while the constraint $\mathcal{C}$ couples positions globally. For example, if $\mathcal{C}$ only accepts valid real numbers, both ``\texttt{1.}'' and ``\texttt{.1}'' are valid length-2 sequences, but a factorized sampler that draws ``\texttt{.}'' at position 1 (consistent with ``\texttt{.1}'') and ``\texttt{.}'' at position 2 (consistent with ``\texttt{1.}'') produces the invalid sequence ``\texttt{..}''. Therefore, the autoregressive remedy of filtering out invalid next tokens does not apply: the validity of position $i$ depends on the tokens being sampled simultaneously at all other positions. Sampling exactly from Equation~(\ref{eq:condition}) therefore requires reasoning 
about the joint distribution induced by the constraint.

\subsection{Tractable Inference over Finite Automata Constraints}
\label{sec:method-chain}

Inference over general constraints is intractable, and for this work, we focus on constraints expressible by finite automata. As reviewed in Section~\ref{sec:back-nfa}, an automaton induces a chain-structured graphical model (Figure~\ref{fig:fa_gm}) and hence a distribution $\pfa(x^0_{1:L})$ over length-$L$ sequences whose support is exactly the constrained language, i.e.,
$\pfa(x^0_{1:L}) > 0$ if and only if $x^0_{1:L}\in \mathcal{C}$.

As described in Section~\ref{sec:back-nfa}, if the finite automaton is deterministic then $\pfa$ is uniform over $\mathcal{C}$. Sampling from the constrained posterior in Equation~(\ref{eq:condition}) therefore reduces to sampling from the product of mean-field and constrained distribution:
\begin{equation}
    \label{eq:product}
    \hat{x}^0 \;\sim\;\left(\prod_{i=1}^{L} p_\theta(x_i^0 \given x^t)\right) \cdot \pfa(x^0_{1:L})
    \;\propto\;
    \text{Equation~(\ref{eq:condition})}
\end{equation}
which is exactly the desired constrained mean-field distribution. 
For non-deterministic finite automata~(NFAs), the product construction above remains tractable, but $\pfa(x^0_{1:L})$ now weights each accepted sequence by the number of accepting paths that generate it; sampling from this weighted distribution becomes a soft proxy for $p_\theta(x^0 \given x^t,\, \mathcal{C})$.

\paragraph{Tractable product construction.}
Given two distributions $p_\theta$ and $\pfa$, our goal is to sample from
$
g(x^0_{1:L}) \propto p_\theta(x^0_{1:L})\cdot \pfa(x^0_{1:L}).
$
The product of two arbitrary distributions is in general intractable.
In our setting, however, $p_\theta$ is fully factorized (mean-field) and
$\pfa$ is chain-structured, thus their product admits a tractable representation with the same chain structure as $\pfa$ in Figure~\ref{fig:fa_gm}~\citep{shen2016tractable, VergariNeurIPS21}.

Concretely, the product is obtained by reweighting each individual emission factor of
$\pfa$ with the corresponding mean-field term of $p_\theta(x^0 \given x^t)\!=\!\prod_{i=1}^{L} p_\theta(x_i^0 \given x^t)$, while leaving
the transition factors unchanged. This yields a new chain-structured model
(Figure~\ref{fig:fa_gm}(b)) with factors
\begin{equation*}
\left\{
\begin{aligned}
g (z_1 = e)
&:= \pfa(z_1 = e ),\\
g (z_i = e' \mid z_{i-1} = e)
&:= \pfa(z_i = e' \mid z_{i-1} = e),\\
g(x_i = v \mid z_i = e)
&:= \pfa(x_i = v \mid z_i = e)\, \cdot p_\theta(x_i = v\given x^t) \\
\end{aligned}
\right.
\end{equation*}
Note that these reweighted factors are unnormalized and hence do not define conditional probabilities on their own; locally normalizing each factor would yield a distribution that is different from $g(x^0_{1:L})$. Instead, we normalize $g$ globally over length-$L$ sequences to preserve the original distribution. 
As in standard chain-structured graphical models~\cite{rabiner1986introduction}, to sample from $g$, we first perform a backward pass to compute the required normalization constants for each factor~\cite{pearl1982reverend} and then draw samples via ancestral sampling.

\paragraph{Time complexity.} 
The backward message-passing costs $\mathcal{O}(L(|\mathcal{E}|^2 + |\mathcal{E}||\mathcal{V}|))$ FLOPs: the $\mathcal{O}(L|\mathcal{E}|^2)$ term comes from the message propagation from $z_t$ to $z_{t-1}$ and the $\mathcal{O}(L|\mathcal{E}||\mathcal{V}|)$ term arises from emission reweighting using LM probabilities. 
The forward pass adds $\mathcal{O}(L(|\mathcal{E}| + |\mathcal{V}|))$ FLOPs of work, sampling $z_i$ and $x_i$ for $1\!\leq\!i\!\leq\!L$. In practice, we use the state-space formulation, reducing the total number of FLOPs to $\mathcal{O}(L(|\mathcal{S}|^2 + |\mathcal{E}||\mathcal{V}|))$ by performing message passing over the states of the finite automaton instead of the edges. Details in Appendix~\ref{app:forward-backward}.

\paragraph{Remasking confidence and diffusion sampling.}
\dlms select which tokens to commit at each step using a per-position confidence score, typically the unconstrained marginal $p_\theta(x_i^0 \mid x^t)$. Under a constraint $\mathcal{C}$, the appropriate score is the constrained marginal $p_\theta(x_i^0 = \hat{x}_i \mid x^t, \mathcal{C})$, which reflects how strongly the constraint and the model jointly prefer $\hat{x}_i$ at position $i$. 

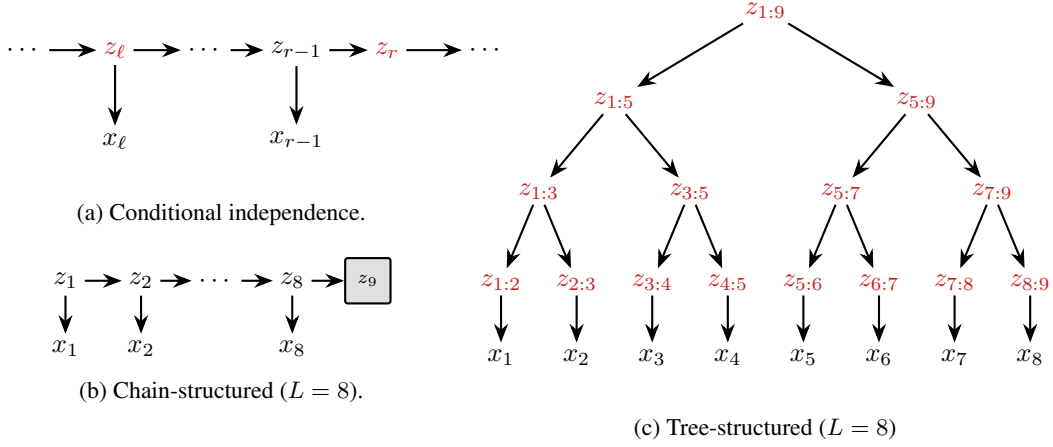
\begin{figure}[t]
    \centering
    \begin{minipage}[c]{0.42\linewidth}
        \centering
        \begin{subfigure}[b]{\linewidth}
            \centering
            \begin{tikzpicture}[
    every node/.style={font=\normalsize, inner sep=2.5pt},
    arr/.style={-{Stealth[length=2.5mm]}, thick},
    sep/.style={spanred, font=\normalsize},
]

\def\dx{1.2}
\def\yEmit{0}
\def\yChain{1.2}

\node       (zd1) at (0*\dx, \yChain) {$\cdots$};
\node[sep]  (zi)  at (1*\dx, \yChain) {$z_\ell$};
\node       (zd2) at (2*\dx, \yChain) {$\cdots$};
\node       (zj)  at (3*\dx, \yChain) {$z_{r-1}$};
\node[sep]  (zjplus) at (4*\dx, \yChain) {$z_{r}$};
\node       (zd3) at (5.1*\dx, \yChain) {$\cdots$};

\node  (xi) at (1*\dx, \yEmit) {$x_\ell$};
\node  (xj) at (3*\dx, \yEmit) {$x_{r-1}$};

\draw[arr] (zd1) -- (zi);
\draw[arr] (zi)  -- (zd2);
\draw[arr] (zd2) -- (zj);
\draw[arr] (zj)  -- (zjplus);
\draw[arr] (zjplus)  -- (zd3);

\draw[arr] (zi) -- (xi);
\draw[arr] (zj) -- (xj);

\end{tikzpicture}     %
            \caption{Conditional independence.}
            \label{fig:hmm-cond-indep}
        \end{subfigure}\\[1em]
        \begin{subfigure}[b]{\linewidth}
            \centering
            \begin{tikzpicture}[
    every node/.style={font=\normalsize, inner sep=2.5pt},
    arr/.style={-{Stealth[length=2.5mm]}, thick},
]

\node          (z1)   at (0,   1.6) {$z_1$};
\node          (z2)   at (1.0, 1.6) {$z_2$};
\node          (zd)   at (2.0, 1.6) {$\cdots$};
\node          (zL)   at (3.0, 1.6) {$z_8$};
\node[draw, fill=black!12, thick, rounded corners=1pt, minimum size=6mm,
      inner sep=1pt] (zLp1) at (4, 1.6) {\scriptsize $z_{9}$};

\node (x1) at (0,   0.7) {$x_1$};
\node (x2) at (1.0, 0.7) {$x_2$};
\node (xL) at (3.0, 0.7) {$x_8$};

\draw[arr] (z1) -- (z2);
\draw[arr] (z2) -- (zd);
\draw[arr] (zd) -- (zL);
\draw[arr] (zL) -- (zLp1);

\draw[arr] (z1) -- (x1);
\draw[arr] (z2) -- (x2);
\draw[arr] (zL) -- (xL);

\end{tikzpicture}             %
            \caption{Chain-structured ($L=8$).}
            \label{fig:hmm-l8}
        \end{subfigure}
    \end{minipage}%
    \hfill
    \begin{subfigure}[c]{0.55\linewidth}
        \centering
        \begin{tikzpicture}[
    every node/.style={font=\normalsize, inner sep=2.5pt},
    line/.style={thick},
    arr/.style={-{Stealth[length=2.5mm]}, thick},
    internal/.style={spanred, font=\normalsize\itshape, inner sep=2pt},
]

\node[internal] (zroot) at (3.5, 5.0) {$z_{1:9}$};

\node[internal] (z15) at (1.5, 3.8) {$z_{1:5}$};
\node[internal] (z59) at (5.5, 3.8) {$z_{5:9}$};

\node[internal] (z13) at (0.5, 2.6) {$z_{1:3}$};
\node[internal] (z35) at (2.5, 2.6) {$z_{3:5}$};
\node[internal] (z57) at (4.5, 2.6) {$z_{5:7}$};
\node[internal] (z79) at (6.5, 2.6) {$z_{7:9}$};

\node[internal] (z12) at (0, 1.4) {$z_{1:2}$};
\node[internal] (z23) at (1, 1.4) {$z_{2:3}$};
\node[internal] (z34) at (2, 1.4) {$z_{3:4}$};
\node[internal] (z45) at (3, 1.4) {$z_{4:5}$};
\node[internal] (z56) at (4, 1.4) {$z_{5:6}$};
\node[internal] (z67) at (5, 1.4) {$z_{6:7}$};
\node[internal] (z78) at (6, 1.4) {$z_{7:8}$};
\node[internal] (z89) at (7, 1.4) {$z_{8:9}$};

\node (x1) at (0, 0.4) {$x_1$};
\node (x2) at (1, 0.4) {$x_2$};
\node (x3) at (2, 0.4) {$x_3$};
\node (x4) at (3, 0.4) {$x_4$};
\node (x5) at (4, 0.4) {$x_5$};
\node (x6) at (5, 0.4) {$x_6$};
\node (x7) at (6, 0.4) {$x_7$};
\node (x8) at (7, 0.4) {$x_8$};

\draw[arr] (zroot) -- (z15);
\draw[arr] (zroot) -- (z59);

\draw[arr] (z15) -- (z13);
\draw[arr] (z15) -- (z35);
\draw[arr] (z59) -- (z57);
\draw[arr] (z59) -- (z79);

\draw[arr] (z13) -- (z12);
\draw[arr] (z13) -- (z23);
\draw[arr] (z35) -- (z34);
\draw[arr] (z35) -- (z45);
\draw[arr] (z57) -- (z56);
\draw[arr] (z57) -- (z67);
\draw[arr] (z79) -- (z78);
\draw[arr] (z79) -- (z89);

\draw[arr] (z12) -- (x1);
\draw[arr] (z23) -- (x2);
\draw[arr] (z34) -- (x3);
\draw[arr] (z45) -- (x4);
\draw[arr] (z56) -- (x5);
\draw[arr] (z67) -- (x6);
\draw[arr] (z78) -- (x7);
\draw[arr] (z89) -- (x8);

\end{tikzpicture}
        \caption{Tree-structured ($L=8$)}
        \label{fig:segtree-hmm}
    \end{subfigure}
    
    \caption{Chain-structured graphical model and its tree-structured equivalent. \textbf{(a)} Given boundary states $z_\ell$ and $z_r$, $x_{\ell:r-1}$ is independent of the rest, the property exploited by the tree decomposition. \textbf{(b)} Naive sequential chain. \textbf{(c)} Tree-structured equivalent enabling $\mathcal{O}(\log L)$ parallel sampling.}
    \label{fig:hmm-vs-segtree}
\end{figure}

In diffusion models, this sampling procedure must be repeated at every denoising step; as a result, the $\mathcal{O}(L)$ sequential depth compounds across steps and becomes a major bottleneck. To address this, we introduce a log-depth sampling algorithm that reduces the sequential dependence to $\mathcal{O}(\log L)$. 

\subsection{Efficient Log-depth Tree Sampling}
\label{sec:method-tree}

Standard ancestral sampling, as well as forward/backward message passing, is sequential for the chain-structured graphical model in Figure~\ref{fig:hmm-l8}: each $z_i$ depends on $z_{i-1}$, so we must proceed through all $L$ states iteratively, even though the computation at each step is small (on the order of $10^6$ FLOPs). To expose more parallelism and fully leverage modern GPUs, we apply depth-reduction techniques for arithmetic circuits~\citep{valiant1981fast,zhang2025restructuring}, reducing the number of sequential steps from $\mathcal{O}(L)$ to $\mathcal{O}(\log L)$.

The key observation is that, for any states $l\!<\!m\!<\!r$, conditioning on the two boundary states $\zl$ and $\zr$ makes the tokens $\xlm$ and $\xmr$ conditionally independent given the midpoint state $\zm$~\citep{pearl1988probabilistic}:
\begin{align}
    p(\xlr\!\given\!\zl, \zr) \;\propto\; \sum_{\zm}
    \underbrace{p(\zm\!\given\!\zl, \zr)}_\text{midpoint}  \cdot\
     \underbrace{p(\xlm\!\given\!\zl, \zm)}_\text{left half}  \cdot\ \underbrace{p(\xmr\!\given\!\zm, \zr)}_{\text{right half}},\;\; \ell <m<r
     \label{eq:tree-recursion}
\end{align}
where we typically take $m=(\ell+r)/2$.
For example, in Figure~\ref{fig:segtree-hmm}, conditioning on $z_1$ and $z_9$ makes $x_{1:4}$ and $x_{5:8}$ conditionally independent given $z_5$.
This enables a divide-and-conquer sampling algorithm. Instead of sampling the latent states sequentially as $z_1, x_1, z_2, x_2, \ldots, z_L, x_L$, we first sample the midpoint latent state $z_{\lfloor L/2 \rfloor + 1}$ given the boundary states. The left and right subproblems are then conditionally independent and can therefore be sampled in parallel. Recursively applying this procedure yields a sampler of depth $\log L$, with all subproblems at each recursion level processed in parallel (Figure~\ref{fig:segtree-hmm}). The base case samples all tokens in parallel from the emission factors $p(x_i\given z_i,z_{i+1})\propto p(x_i\given z_i)\cdot p(z_{i+1}\given z_i)$.
The conditional coefficients $p(\zm \mid \zl, \zr)$ in Equation~(\ref{eq:tree-recursion}) are not given by the model directly, but can be computed by recursion as well from the relation:
\begin{equation}
    p(\zm \mid \zl, \zr) \;\propto\; p(\zm \mid \zl) \cdot p(\zr \mid \zm),
    \label{eq:midpoint-conditional}
\end{equation}
where the right-hand side are multi-step transition probabilities, and all of which can be computed in $\log L$ steps. The full algorithm consists of two phases: a \emph{bottom-up} pass that precomputes the multi-step transitions, and a \emph{top-down} divide-and-conquer pass that samples midpoint states level by level. Once all hidden states $z_{1:L}$ are sampled, the tokens $x_{1:L}$ are sampled independently in parallel from $p(x_i \mid z_i, z_{i+1})$. Pseudocode is given in Algorithm~\ref{alg:parallel-sample}; the proof is in Appendix~\ref{app-sec:parallel-sample}.

\begin{algorithm}[t]
\caption{Parallel constrained sampling from $p_\theta(x^0 \mid x^t, \mathcal{C})$}
\label{alg:parallel-sample}
\begin{algorithmic}[1]
\Require Diffusion predictions $p_\theta(x^0 \mid x^t)$, constraint $\mathcal{C}$, length $L$
\Ensure Sample $x^0 \sim p_\theta(x^0 \mid x^t, \mathcal{C})$
\For{$k = 0, 1, \ldots, \log_2 L$} \Comment{compute messages bottom-up}
    \State Compute $2^k$-step transitions $\{p(z_{i + 2^k} \mid z_i)\}_{i}$ from $2^{k-1}$-step transitions
\EndFor
\State Sample boundary states $(z_1, z_{L+1})$ from their joint marginal \Comment{sample top-down}
\State \Return $\textsc{Sample}(1, L+1, z_1, z_{L+1})$
\Statex
\Procedure{Sample}{$\ell, r, \zl, \zr$}
    \If{$r - \ell \leq 1$}
    \State \Return sample $x_\ell \sim p(x_\ell \mid \zl, \zr)$ \Comment{base case}
    \EndIf
    \State $m \gets \lfloor (\ell + r) / 2 \rfloor$
    \State Sample $\zm \sim p(\zm \mid \zl, \zr) \propto p(\zm \mid \zl) \cdot p(\zr \mid \zm)$
    \State \textbf{in parallel:} 
    \State \quad $x_{\ell:m-1} \gets \textsc{Sample}(\ell, m, \zl, \zm)$
    \State \quad $x_{m:r-1} \gets \textsc{Sample}(m, r, \zm, \zr)$
    \State \Return $\text{concat}(x_{\ell:m-1},\, x_{m:r-1})$
\EndProcedure
\end{algorithmic}
\end{algorithm}

\paragraph{Time complexity.} The bottom-up parallel pass computes all multi-step transitions via Equation~(\ref{eq:midpoint-conditional}), costing $\mathcal{O}(L(|\mathcal{E}|^3 + |\mathcal{E}||\mathcal{V}|))$ FLOPs in $\mathcal{O}(\log L)$ depth: the $\mathcal{O}(L|\mathcal{E}|^3)$ term comes from $|\mathcal{E}| \times |\mathcal{E}|$ matrix multiplications, while the $\mathcal{O}(L|\mathcal{E}||\mathcal{V}|)$ emission reweighting term is the same as in ancestral sampling. The top-down sampling pass samples the hidden states $z_{1:L+1}$ in $\mathcal{O}(L|\mathcal{E}|)$ work and $\mathcal{O}(\log L)$ depth via the recursion tree, after which all tokens $x_{1:L}$ are sampled fully in parallel at $\mathcal{O}(L|\mathcal{V}|)$ work. Similar to ancestral sampling, the total number of FLOPs can be reduced to $\mathcal{O}(L(|\mathcal{S}|^3 + |\mathcal{E}||\mathcal{V}|))$ in $\mathcal{O}(\log L)$ depth when performing message passing over states.

\begin{wrapfigure}{r}{0.33\linewidth}
  \vspace{-2em}
  \centering
  \includegraphics[width=\linewidth]{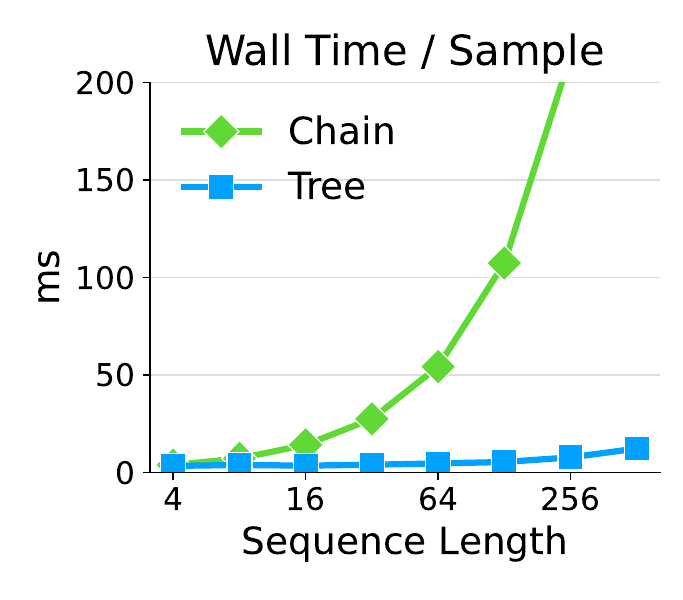}
  \caption{{Runtime comparison for chain (ancestral) and tree (log-depth) sampling on a DFA.}}
  \label{fig:tree-vs-chain}
  \vspace{-1em}
\end{wrapfigure}
\paragraph{Parallelization vs. FLOPs tradeoff.} Compared to ancestral sampling, log-depth sampling increases the message-passing FLOPs by a factor of $|\mathcal{S}|$, but reduces the depth of the computation graph from $\mathcal{O}(L)$ to $\mathcal{O}(\log L)$. 
This tradeoff is favorable on modern GPUs: the additional FLOPs are spent on larger matrix operations, while the number of sequential GPU calls is reduced from $\mathcal{O}(L)$ to $\mathcal{O}(\log L)$. 
Since each step of ancestral sampling performs relatively little work, its runtime is often dominated by synchronization and kernel-launch overhead, while the log-depth sampler exposes substantially more parallelism at each layer, leading to higher GPU utilization and lower wall-clock latency.

\section{Experiments}
\label{sec:exp}

\subsection{Tasks, Models and Setups}
\paragraph{Tasks} 
We evaluate our method on five tasks across six datasets. Constraints are encoded as finite automata (DFA or NFA); full prompts and example outputs are provided in Appendix~\ref{app-sec:task}.

\begin{itemize}[leftmargin=*]
\item \textbf{Function Calling (xLAM, BFCL).} 
\emph{Task:} Given a user query and a set of available function schemas, generate the corresponding function call.
\emph{Data:} xlam-function-calling-60k (xLAM)~\citep{liu2024apigen} and the single-turn split of the Berkeley Function Calling Leaderboard (BFCL)~\citep{patil2025the}, including both non-live and live settings. 
\emph{Metric:} Exact match against the ground-truth function call. 
\emph{Constraint:} A DFA enforcing that the output conforms to the given function schema. We evaluate two output formats: \textit{Python format} (the BFCL standard, e.g., {\verb|get_weather(location="NYC")|}) and \textit{JSON format} (e.g., \verb|{"name": "get_weather", "location": "NYC"}|).

\item \textbf{Sudoku.} 
\emph{Task:} Solve a 4$\times$4 Sudoku puzzle~\citep{ye2024diffusion}.
\emph{Metric:} Exact match against valid Sudoku solutions. Averaged across number of prefilled cells $\{4,\ldots,12\}$.
\emph{Constraint:} A DFA enforcing the 4$\times$4 output format and preserving prefilled input cells; row, column, and box constraints are not enforced by the DFA, which would require an exponentially larger DFA.

\item \textbf{Countdown.} 
\emph{Task:} Reach a target number using arithmetic operations on three given numbers~\citep{ye2024diffusion}.
\emph{Metric:} Whether the expression evaluates to the target.
\emph{Constraint:} A DFA enforcing $A \mathbin{\text{op}} B \!=\! C$ at each step, with operands from the input or prior results and the final result equal to the target.

\item \textbf{Math reasoning (GSM-Symbolic).} 
\emph{Task:} Solve grade-school math word problems.
\emph{Data:} GSM-Symbolic~\citep{mirzadeh2025gsmsymbolic}, where numerical values and names are replaced by symbolic variables.
\emph{Metric:} Exact match on the final answer.
\emph{Constraint:} A DFA enforcing that symbolic expressions appear only within \texttt{<<} and \texttt{>>}, while the remaining text is unrestricted natural language.

\item \textbf{Text-to-SQL (Spider).}
\emph{Task:} Given a user query and a database schema, generate the corresponding SQL query.
\emph{Data:} Spider~\citep{yu2018spider}.
\emph{Metric:} Execution accuracy. A query is correct if it returns the same results as the ground truth on the test database.
\emph{Constraint:} An NFA encodes SQL grammar restricted to the columns and tables in the provided schema.

\end{itemize}

\paragraph{Models} 
We evaluate two open-weight \dlms, \llada-8B~\citep{nie2025large} and \dream-7B~\citep{ye2025dream}, each with \emph{Base} and \emph{Instruct} variants. We use the \emph{Base} variants for Sudoku and Countdown and the \emph{Instruct} variants for the remaining tasks~\citep{ye2025dream}.
For position selection, \llada selects the highest-probability positions, while \dream selects the lowest entropy ones. Our method computes marginal probabilities under the constrained distribution at each step and uses them as the confidence signal for both models.

\subsection{Main Results}
\begin{table*}[t]
\centering
\caption{Comparing unconstrained ({Base}) and constrained ({Ours}) accuracy (\%) under greedy decoding and stochastic sampling on \dream-7B and \llada-8B. {CS} is the baseline constraint satisfaction rate; our method guarantees $100\%$ constraint satisfaction by construction. For function calling, we report both \emph{JSON} and \emph{Python} formats; for BFCL, L/NL denote Live and Non-Live settings.}
\label{tab:main-results}
\small
\setlength{\tabcolsep}{4pt}
\renewcommand{\arraystretch}{1.15}

\begin{tabular}{@{}l ccc |ccc| ccc| ccc@{}}
\toprule
& \multicolumn{6}{c}{\textbf{\dream-7B}} & \multicolumn{6}{c}{\textbf{\llada-8B}} \\
\cmidrule(lr){2-7} \cmidrule(lr){8-13}
& \multicolumn{3}{c}{\textbf{Greedy}} & \multicolumn{3}{c}{\textbf{Sampling}} 
& \multicolumn{3}{c}{\textbf{Greedy}} & \multicolumn{3}{c}{\textbf{Sampling}} \\
\textbf{Task / Method} & CS & Base & Ours & CS & Base & Ours 
              & CS & Base & Ours & CS & Base & Ours \\
\midrule
\rowcolor{gray!12}
\multicolumn{13}{@{}l}{\textit{Function Calling}} \\
xLAM \emph{(JSON)}   
& 90.1  & 74.6 & \textbf{76.6} & 59.2  & 50.5 & \textbf{75.7} 
& 90.2 & 67.8 & \textbf{69.0} & 86.4  & 64.2 & \textbf{68.4} \\

BFCL L \emph{(JSON)}
& 86.7  & 63.9 & \textbf{71.5} & 31.2 & 22.3 & \textbf{69.0} 
& 89.5 & 59.6 & \textbf{65.6} & 84.7 & 56.6 & \textbf{63.3} \\

BFCL NL \emph{(JSON)}
& 90.3 & 79.9 & \textbf{82.4} & 47.1 & 40.8 & \textbf{82.6} 
& 93.5 & 81.7 & \textbf{85.5} & 92.9 & 80.5 & \textbf{82.8} \\

xLAM \emph{(Python)}
& 70.3 & 56.1 & \textbf{76.4} & 25.6 & 20.9 & \textbf{76.8}
& 77.3 & 60.1 & \textbf{71.4} & 69.4 & 54.2 & \textbf{70.8} \\

BFCL L \emph{(Python)}
& 36.0 & 22.4 & \textbf{69.7} & 4.4 &  2.8 & \textbf{67.8} 
& 74.5 & 50.3 & \textbf{64.0} & 68.3 & 46.8 & \textbf{63.0} \\

BFCL NL \emph{(Python)}
& 44.4 & 36.3 & \textbf{78.3} & 8.4 &  7.2 & \textbf{76.2} 
& 83.8 & 71.4 & \textbf{83.8} &  76.1 & 62.3 & \textbf{81.4}	  \\

\rowcolor{gray!12}
\multicolumn{13}{@{}l}{\textit{Planning}} \\

Sudoku        
& 38.0  & 36.4 & \textbf{92.4}  & 7.6 & 4.7 & \textbf{91.9} 
& 20.7 & 19.8 & \textbf{79.6} & 19.8  & 18.7 & \textbf{77.9} \\

Countdown   
& 53.6  & 53.6 & \textbf{100.} & 35.2 & 35.2 & \textbf{100.}
& 53.8 & 53.8 & \textbf{100.} & 53.5  & 53.5 & \textbf{100.} \\

\rowcolor{gray!12}
\multicolumn{13}{@{}l}{\textit{Code \& Math}} \\

Spider        
& 72.2 & 52.7 & \textbf{53.0}  & 19.3  & 15.5 & \textbf{52.1}  
& 73.5 & 53.2 & \textbf{53.6}           
& 69.1 & 49.8 &  \textbf{53.3}         \\

GSM-Symbolic  
& 80.2 & 39.9 & \textbf{40.2}  & 54.6 & 15.0 & \textbf{38.0}  
& 81.5 & \textbf{49.6} & 49.2          
& 80.1 & 46.9 & \textbf{49.2} \\

\bottomrule
\end{tabular}
\end{table*}

Table~\ref{tab:main-results} reports baseline and constrained accuracy across five tasks under two decoding regimes: greedy decoding (temperature 0) and stochastic sampling (temperature 1), along with the baseline's constraint satisfaction rate. Our method guarantees $100\%$ constraint satisfaction by construction.

\textbf{Stochastic sampling collapses unconstrained \dlms; constrained decoding recovers near-greedy accuracy.} Without constraints, switching from greedy ($T{=}0$) to stochastic decoding ($T{=}1$) destroys structured-output performance: \dream's BFCL-Live (JSON) accuracy drops from 63.9\% to 22.3\%. The collapse is largely driven by constraint satisfaction: under sampling, far more outputs fail to satisfy. Our method improves greedy accuracy across tasks and, more interestingly, preserves accuracy under stochastic sampling where the unconstrained baseline collapses (69.0\%, vs.\ 22.3\% baseline), largely closing the gap to greedy performance.

\textbf{The format gap is structural, not semantic.} Switching from JSON to Python format sharply degrades baseline performance (\dream BFCL-Live: 63.9\% to 22.4\%; \llada: 59.6\% to 50.3\%), suggesting that these models know \emph{what function} to call but struggle with \emph{how} to format it. Our method largely closes this gap (\dream: 71.5\% to 69.7\%; \llada: 65.6\% to 64.0\%) by supplying the missing structure, enabling generalization across formats.

\begin{wraptable}{r}{0.4\linewidth}
\centering
\footnotesize
\vspace{-1.5em}
\caption{Ablation on remasking. {Mf} uses base model's mean-field prediction as confidence; {Mar} uses marginals under constrained distribution as confidence.}
\setlength{\tabcolsep}{2pt}
\begin{tabular}{l ccc | ccc}
\toprule
& \multicolumn{3}{c}{\textbf{\dream}} & \multicolumn{3}{c}{\textbf{\llada}}
\\
\textbf{Task} & Base & {Mf} & {Mar} & Base & {Mf} & {Mar}  \\
\midrule
xLAM \emph{(JS)}        & 74.6  & \textbf{77.5}  & 76.6           & 67.8    & \textbf{69.1}             & 69.0             \\
xLAM \emph{(Py)}      & 56.1  & 68.4           & \textbf{76.4}  & 60.1    & 58.7             & \textbf{71.4}             \\
BFCL-L \emph{(JS)}   & 63.9 &    69.4       & \textbf{71.5} & 59.6 & 64.9	 & \textbf{65.6}          \\
BFCL-L \emph{(Py)} & 22.4 &   61.5       & \textbf{69.7} & 50.3 & 61.5 & \textbf{64.0}          \\
\bottomrule
\end{tabular}

    \label{tab:ablation-conf}
\end{wraptable}

\textbf{Improvements come from both filtering invalid sequences and remasking under the constrained distribution.}
We compare two remasking strategies. ({Mf}) uses the baseline's mean-field prediction as confidence to select positions. (Mar) recomputes confidence as marginal probabilities under the constrained distribution, changing the decoding order itself.
Table~\ref{tab:ablation-conf} shows that Mar outperforms {Mf}, with the gap largest when the unconstrained baseline performs poorly. This pattern suggests ``filtering'' alone cannot find high-probability valid sequences when the model's mean-field is misaligned with the constraint; recomputing marginals provides the missing signal.

\subsection{Ablation and Analysis}

\textbf{Runtime overhead is small relative to accuracy gains.}
We measure per-sample latency, throughput, and accuracy on a 200-example subset of BFCL. As shown in Table~\ref{tab:runtime}, our tree variant improves accuracy by 12 points over the unconstrained baseline while incurring only 4\% wall-clock overhead. 
This shows that finite automaton inference is lightweight compared to \dlm and parallelizes well.

\textbf{Scaling with fewer generation steps.}
Figure~\ref{fig:bfcl-steps} evaluates \dream and \llada on the BFCL simple split as denoising steps decrease from 256 to 16. Baseline performance drops sharply with fewer steps, whereas our method remains much more robust. With more parallelism, unconstrained sampling produces more invalid tokens, whereas our method maintains valid structure and preserves accuracy.

\begin{table}[h]
\centering
\caption{{Runtime, accuracy, and throughput scaling on 200 BFCL examples (\dream)}. Latency uses batch size 1; throughput reports tokens/s as batch size increases. Measured on an A6000. Overhead is wall-clock relative to the unconstrained baseline. Mean $\pm$ std over 3 seeds.}
\label{tab:runtime}
\footnotesize
\setlength{\tabcolsep}{4pt}
\renewcommand{\arraystretch}{1.1}
\begin{tabular}{l c  cc  ccc}
\toprule
\multirow{2}{*}{\textbf{Method}} & \multirow{2}{*}{\textbf{Accuracy}} 
& \multicolumn{2}{c}{\textbf{Latency}} 
& \multicolumn{3}{c}{\textbf{Throughput (tok/s)}} \\
\cmidrule(lr){3-4} \cmidrule(lr){5-7}
& 
& \textbf{s/sample} 
& \textbf{Overhead} 
& \textbf{bs=1} 
& \textbf{bs=2} 
& \textbf{bs=4} \\
\midrule
Unconstrained                          
& 67.5 
& 24.26 $\pm$ 0.05 
& ---     
& 10.55 $\pm$ 0.02 
& 10.79 $\pm$ 0.06 
& 11.17 $\pm$ 0.05 \\

Ours (chain)                           
& 78.5 $\pm$ 0.2 
& 51.92 $\pm$ 0.14 
& +114\% 
& 4.93 $\pm$ 0.01 
& 6.88 $\pm$ 0.06 
& 8.62 $\pm$ 0.05 \\

Ours (tree)  
& 79.2 $\pm$ 0.5 
& 25.35 $\pm$ 0.04  
& +4\% 
& 10.10 $\pm$ 0.02 
& 10.21 $\pm$ 0.02 
& 10.44 $\pm$ 0.03 \\

\bottomrule
\end{tabular}
\end{table}
\begin{figure*}[!h]
  \centering
  \includegraphics[width=0.262\linewidth]{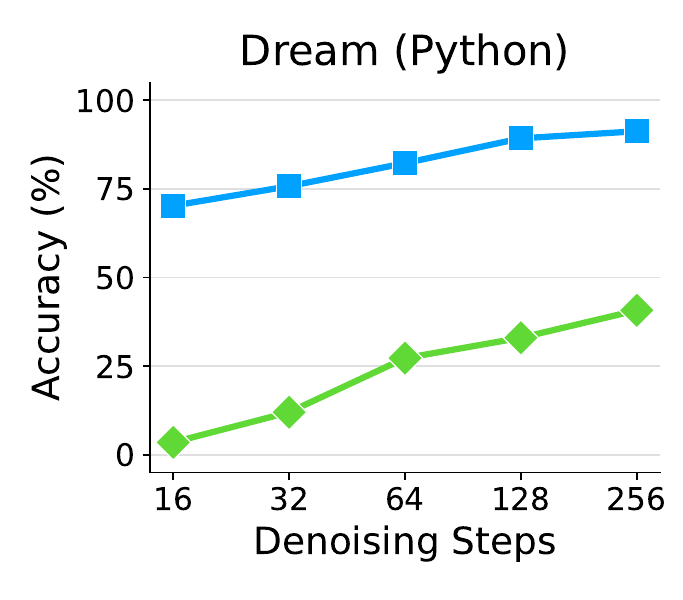}\hfill
  \includegraphics[width=0.236\linewidth]{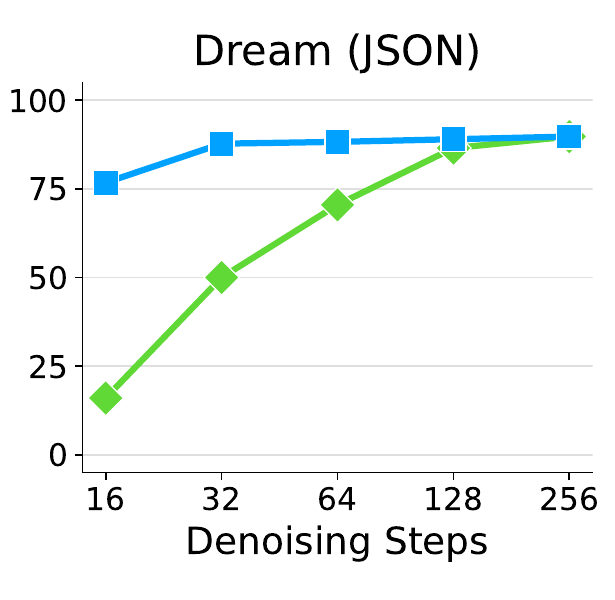}\hfill
  \includegraphics[width=0.236\linewidth]{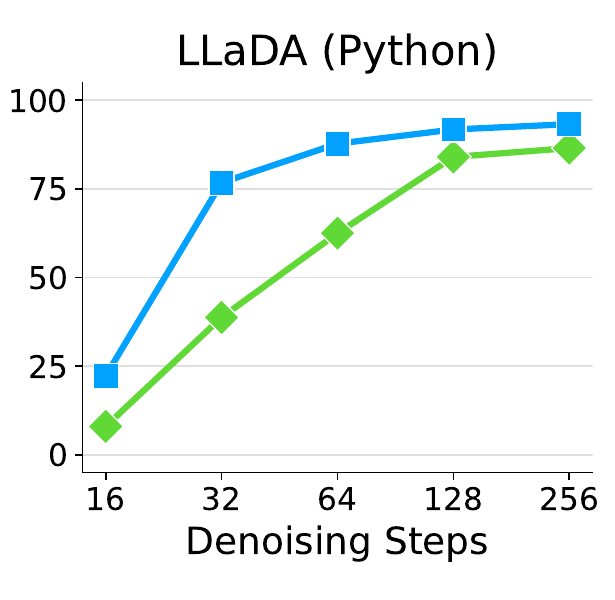}\hfill
  \includegraphics[width=0.236\linewidth]{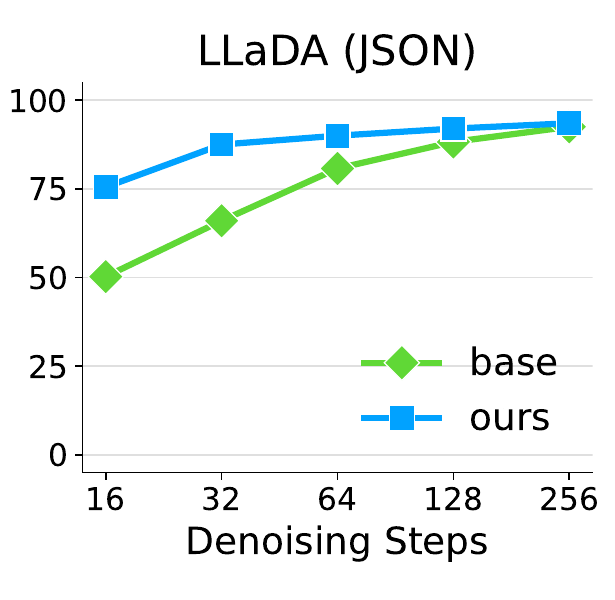}
  \caption{Accuracy vs.\ denoising steps on BFCL Simple split. Our method remains robust as the number of denoising steps decreases, whereas the unconstrained baseline deteriorates sharply.}
  
  \label{fig:bfcl-steps}
\end{figure*}
\begin{table}[!h]
\centering
\caption{BFCL accuracy (\%) comparing diffusion LLMs and autoregressive LLMs (Python format),
with autoregressive results from BFCL leaderboard~\citep{patil2025the}. Our method
substantially narrows the gap.}
\label{tab:bfcl-dlm-vs-ar}
\footnotesize
\setlength{\tabcolsep}{4pt}
\renewcommand{\arraystretch}{1.15}
\begin{tabular}{lcccc|cc}
\toprule
 & \multicolumn{4}{c}{\textbf{Diffusion LLMs (\dlms)}} & \multicolumn{2}{|c}{\textbf{Autoregressive LLMs}} \\
\cmidrule(lr){2-5} \cmidrule(lr){6-7}
\textbf{Setting} 
& \dream & \dream + Ours 
& \llada & \llada + Ours 
& Llama-3.1-8B & Qwen3-8B \\
\midrule

Non-live     & 36.3 & \textbf{78.3} & 71.4 & \textbf{83.8} & 84.0 & 88.6 \\
Live & 22.4 & \textbf{69.7} & 50.3 & \textbf{64.0} & 70.8 & 80.1 \\
\bottomrule
\end{tabular}
\end{table}

\textbf{Comparison with \dlms and autoregressive LLMs of similar size.}
Table~\ref{tab:bfcl-dlm-vs-ar} compares diffusion LLMs with similarly sized autoregressive models on BFCL. Unconstrained diffusion models perform substantially worse and our method substantially narrows the gaps.

\textbf{NFAs vs. DFAs.} The constrained decoding workload spans over three orders of magnitude, from a 21-state \emph{DFA} for Sudoku, to DFAs with ${\sim}$1k states for function calling, and \emph{NFAs} with ${\sim}$10k states and ${\sim}$100k edges for Spider. We use DFAs when their sizes are feasible and fall back to NFAs for Spider, where the equivalent DFA would be exponentially large (Details in Appendix~\ref{app-sec:task}).

\section{Conclusion}
We propose a constrained decoding algorithm for diffusion language models under constraints expressible as finite automata. Viewing finite automata as graphical models, our approach performs exact and tractable sampling/inference over the constrained mean-field posterior at each denoising step, guaranteeing constraint satisfaction. We further propose a log-depth sampling strategy to unlock greater parallelization, yielding low overhead in practice. Our approach substantially improves model performance across a variety of downstream tasks where unconstrained models struggle.

\section{Limitations and Societal Impacts}
\label{sec:limitations}
Our method focuses on finite automata, and does not cover richer constraints such as context-free grammars. Constrained decoding can improve the reliability of LLM systems used for function calling and other structured-output applications. 
However, if misused, this method could enable the generation of syntactically valid but malicious tool calls or database queries, increasing the risk of unintended or harmful system actions.

\newpage
{
    \small
    \bibliographystyle{plainnat}
    \bibliography{main}
}

\appendix

\section{Forward-Backward Inference for Chain-Structured Graphical Models}
\label{app:forward-backward}

We give the standard forward--backward algorithm for the chain-structured
graphical model defined in Section~\ref{sec:back-nfa} (Figure~\ref{fig:fa_gm}).
We first present the \emph{edge-space} formulation, which mirrors the
edge-valued latents $z_t \in \mathcal{E}$ used in the main text, and then
aggregate edge messages by their endpoint state to obtain the equivalent
\emph{state-space} formulation, which has lower cost. This is a reformulation of the standard hidden Markov models backward message passing and sampling algorithm~\cite{rabiner1986introduction}.

\subsection{Edge Space}
We use the unnormalized factors of Section~\ref{sec:back-nfa}:
\begin{equation*}
\left\{
\begin{aligned}
g(x_i = v \mid z_i = e) &:= \pfa(x_i = v \mid z_i = e)\, \cdot p_\theta(x_i = v\given x^t), \\
g(z_1=e) &:= \pfa(z_1 = e),\\
g (z_i = e' \mid z_{i-1} = e)
&:= \pfa(z_i = e' \mid z_{i-1} = e).
\end{aligned}
\right.
\end{equation*}
We sample from $g$ in two passes.

\paragraph{Backward pass.} Define backward messages
$$\beta_t : \mathcal{E} \to \mathbb{R}_{\geq 0} \text{\quad for \quad} t = 1, \ldots, L,$$ where
$\beta_t(e)$ is the partition function of the suffix
$(z_{t+1:L}, x_{t+1:L})$ under $g$ given $z_t = e$ restricted to constraint-satisfying sequences. With the base case absorbing the accept condition
\begin{equation*}
    \beta_L(e) \;=\; \ind{\dst(e) \in \mathcal{F}},
\end{equation*}
the recurrence for $t = L{-}1, L{-}2, \ldots, 1$ is
\begin{equation*}
    \beta_t(e) \;=\; \sum_{e' \in \mathcal{E}}\,
                     g(z_{t+1}{=}e' \mid z_t{=}e)\,
                     \left( \sum_{v \in \mathcal{V}}\,
                            g(x_{t+1}{=}v \mid z_{t+1}{=}e') \right)
                     \beta_{t+1}(e').
\end{equation*}
The structural constraints live inside $g$: $g(z_{t+1} \mid z_t)$ vanishes
unless the edges chain end-to-end, and $g(x_{t+1} \mid z_{t+1})$ vanishes
outside the edge label set, so the sums above effectively range over chained
successors and label-compatible tokens.

\paragraph{Forward sampling pass.} We start by sampling $(z_1, x_1)$
jointly from
\begin{equation}
    \mathbb{P}\bigl(z_1 = e,\; x_1 = v\bigr) \;\propto\;
        g(z_1 = e)\,
        g(x_1 = v \mid z_1 = e)\,
        \beta_1(e).
\end{equation}
For each subsequent step $t = 2, \ldots, L$, conditional on the previously
sampled $z_{t-1}$, jointly sample $(z_t, x_t)$ with probability
\begin{equation}
    \mathbb{P}\bigl(z_t = e,\; x_t = v\bigr) \;\propto\;
        g(z_t = e \mid z_{t-1})\,
        g(x_t = v \mid z_t = e)\,
        \beta_t(e),
\end{equation}
with the convention $g(z_1 \mid z_0) := g(z_1)$ for the initial step. The
factor $\beta_t(e)$ reweights each candidate edge by the probability that the
suffix from $\dst(e)$ extends into an accepting state. By construction, the
resulting $x_{1:L}$ is the marginal of $g/Z$ where $Z$ is the partition restricted to accepting paths.

\paragraph{Time complexity.} 
The backward message-passing pass costs $\mathcal{O}(L(|\mathcal{E}|^2 + |\mathcal{E}||\mathcal{V}|))$ FLOPs: the $\mathcal{O}(L|\mathcal{E}|^2)$ term comes from edge-to-edge message propagation, while the $\mathcal{O}(L|\mathcal{E}||\mathcal{V}|)$ term arises from emission reweighting using LM probabilities. 
The forward pass for sampling $z_i$ and $x_i$ with $1\leq i \leq L$ adds $\mathcal{O}(L(|\mathcal{E}| + |\mathcal{V}|))$ FLOPs of work.

\subsection{State Space}
The same chain admits an equivalent presentation over states. Aggregating edges
by their endpoints yields state-valued latents $s_t \in \mathcal{S}$ for
$t = 0, 1, \ldots, L$ (with $s_0$ pinned to the FA's initial state), and the
joint factors of $g$ become
\begin{equation*}
\left\{
\begin{aligned}
g(s_0) &\;=\; \ind{s_0 = s_0^{\mathrm{init}}}, \\
g(s_t, x_t = v \mid s_{t-1}) &\;=\;
   \!\!\!\!\sum_{e :\, \src(e) = s_{t-1},\, \dst(e) = s_t}\!\!\!\!
       g(x_t = v \mid z_t = e),
\end{aligned}
\right.
\end{equation*}
together with the accept indicator $\ind{s_L \in \mathcal{F}}$. The joint
transition factor aggregates contributions from all edges that traverse the
same state pair, naturally handling NFAs where multiple edges may connect
$s_{t-1}$ to $s_t$. It is convenient to define the per-position transition
matrix
\begin{equation*}
M_t \in \mathbb{R}_{\geq 0}^{|\mathcal{S}| \times |\mathcal{S}|},
\qquad
M_t(s, s') \;:=\; \sum_{v \in \mathcal{V}}\,
                  g(s_t = s',\, x_t = v \mid s_{t-1} = s),
\end{equation*}
which marginalizes over the emitted token.

\paragraph{Backward pass.} Define
$\beta_t : \mathcal{S} \to \mathbb{R}_{\geq 0}$ for $t = 0, 1, \ldots, L$
(overloading the symbol; the argument type disambiguates state from edge).
With
\begin{equation*}
    \beta_L(s) \;=\; \ind{s \in \mathcal{F}},
    \qquad
    \beta_{t-1}(s) \;=\; \sum_{s' \in \mathcal{S}}\, M_t(s, s')\, \beta_t(s'),
\end{equation*}
this is the standard HMM backward recurrence with transition matrix $M_t$,
related to the edge-space messages by $\beta_t(s) = \beta_t(e)$ for any edge
$e$ with $\dst(e) = s$ (well-defined since edge-space $\beta_t(e)$ depends on
$e$ only through $\dst(e)$).

\paragraph{Forward sampling pass.} Initialize
$s_0 = s_0^{\mathrm{init}}$. For each $t = 1, 2, \ldots, L$, jointly sample
$(s_t, x_t)$ with probability
\begin{equation*}
    \mathbb{P}(s_t = s',\, x_t = v \mid s_{t-1}) \;\propto\;
        g(s_t = s',\, x_t = v \mid s_{t-1})\, \beta_t(s').
\end{equation*}

\paragraph{Time complexity.}
Each backward update is $\mathcal{O}(|\mathcal{S}|^2)$, forming $M_t$
over the vocabulary costs $\mathcal{O}(|\mathcal{E}|\,|\mathcal{V}|)$ per
position, and forward sampling at each position contributes
$\mathcal{O}(|\mathcal{S}| + |\mathcal{V}|)$ via the staged
$(s_t, x_t)$ draw described above. The total state-space cost is
$\mathcal{O}\bigl(L\,(|\mathcal{S}|^2 + |\mathcal{E}|\,|\mathcal{V}|)\bigr)$,
improving over the naive $\mathcal{O}(L\,|\mathcal{E}|^2)$ edge-space cost.

\section{Parallel Sampling Algorithm}
\label{app-sec:parallel-sample}
We show that Algorithm~\ref{alg:parallel-sample} produces samples distributed exactly as the constrained posterior $p(x^0 \mid x^t, \mathcal{C})$, despite using unnormalized conditionals during sampling.
\begin{proposition}\label{prop:exact-sampling}
Let $p(x^0 \mid x^t)$ denote the diffusion model's mean-field predictions, and let $\mathcal{C}$ be a constraint encoded as a finite automaton with state space $\mathcal{S}$. Algorithm~\ref{alg:parallel-sample} returns a sample $x^0$ distributed exactly as
\begin{equation*}
    p(x^0 \mid x^t, \mathcal{C}) \;\propto\; p(x^0 \mid x^t) \cdot \ind{x^0 \in \mathcal{L}(\mathcal{C})}.
\end{equation*}
\end{proposition}
\begin{proof}
We proceed by induction on the recursion depth.
\paragraph{Base case.} At a leaf of the recursion (segment with boundary states $\zl, \zr$ satisfying $r - \ell \leq 1$), the algorithm samples a single token $x_\ell$ from the conditional $p(x_\ell \mid \zl, \zr)$. Since $x_\ell$ ranges over a finite vocabulary, this categorical distribution can be normalized exactly by summing over its support, and the resulting sample is exact.
\paragraph{Inductive step.} Consider a segment $[\ell, r]$ with boundary states $\zl, \zr$ and midpoint $m$. The algorithm performs three steps:
\begin{enumerate}
    \item Sample $\zm$ from $p(\zm \mid \zl, \zr) \propto p(\zm \mid \zl) \cdot p(\zr \mid \zm)$.
    \item Recursively sample $x_{\ell:m-1}$ from the left subproblem with endpoints $(\zl, \zm)$.
    \item Recursively sample $x_{m:r-1}$ from the right subproblem with endpoints $(\zm, \zr)$.
\end{enumerate}
Step 1 is a categorical sampling problem over the finite state space $\mathcal{S}$, and the unnormalized scores can be summed and normalized exactly in $\mathcal{O}(|\mathcal{S}|)$. By the inductive hypothesis, steps 2 and 3 produce exact samples from their respective conditional distributions.
By $d$-separation in the hidden Markov model induced by $p(x^0 \mid x^t, \mathcal{C})$ (Figure~\ref{fig:fa_gm}), conditioning on $\zl, \zm, \zr$ renders $x_{\ell:m-1}$ and $x_{m:r-1}$ independent. Hence the joint distribution over the segment factors as
\begin{equation*}
    p(x_{\ell:r-1}, \zm \mid \zl, \zr) \;=\; p(\zm \mid \zl, \zr) \cdot p(x_{\ell:m-1} \mid \zl, \zm) \cdot p(x_{m:r-1} \mid \zm, \zr),
\end{equation*}
which is precisely the distribution sampled by steps 1--3. Thus the segment-level sample is exact.
\paragraph{Composition.} The boundary states $z_1, z_{L+1}$ at the root of the recursion are sampled from their joint marginal $p(z_1, z_{L+1})$, which is computed exactly from the precomputed skip-transitions. Combined with the inductive argument above, the full sample $x^0$ is distributed exactly as $p(x^0 \mid x^t, \mathcal{C})$.
\paragraph{Remark on unnormalized conditionals.} Although each conditional $p(\zm \mid \zl, \zr)$ is written up to a $(\zl, \zr)$-dependent normalization constant, this poses no issue for exactness: at sampling time, the constant is computed by summing the $|\mathcal{S}|$ unnormalized scores, and the resulting categorical distribution is exact. Different segments use different normalization constants, as they should---each subproblem is conditioned on different endpoint values, and the chain-rule decomposition above ensures these segment-local normalizers compose into the correct joint distribution.
\end{proof}
\section{Experiment Details}
\label{app-sec:task}

\subsection{Generation Configurations} 
Table~\ref{tab:gen-settings} summarizes the generation hyperparameters used for each task. We follow the default decoding configurations released with each model. Generation length and number of denoising steps are matched across models within each task to ensure a fair comparison.
\begin{table}[h]
\centering
\small
\renewcommand{\arraystretch}{1.2}
\caption{Generation settings per task across the evaluated diffusion language models. All runs use temperatures $T=0$ and $T=1$. The {\textbf{Remasking}} column reports the algorithm for both unconstrained and constrained decoding; for constrained runs, confidence values are recomputed from the constrained distribution.}
\begin{tabular}{
@{}
ll
>{\columncolor{gray!10}}c
>{\columncolor{gray!10}}c
>{\columncolor{gray!10}}c
l
@{}}
\toprule
\textbf{Task} & \textbf{Model} & \textbf{Seq Length} & \textbf{Steps} & \textbf{Block Size} & \textbf{Remasking} \\
\midrule

\multirow{2}{*}{BFCL \emph{(Python / JSON)}}
  & \llada-8B-Instruct       & 256 & 128 & 32  & { low\_confidence} \\
  & \dream-v0-Instruct-7B      & 256 & 128 & --- & { entropy } \\
\midrule

\multirow{2}{*}{xLAM \emph{(Python / JSON)}}
  & \llada-8B-Instruct       & 128 & 64  & 32  & { low\_confidence} \\
  & \dream-v0-Instruct-7B      & 128 & 64  & --- & { entropy} \\
\midrule

\multirow{2}{*}{Sudoku 4$\times$4}
  & \llada-8B-Base          & 32  & 32  & 32  & { low\_confidence} \\
  & \dream-v0-Base-7B    & 32  & 32  & --- & { entropy} \\
\midrule

\multirow{2}{*}{Countdown \emph{(cd3)}}
  & \llada-8B-Base          & 32  & 32  & 32  & { low\_confidence} \\
  & \dream-v0-Base-7B    & 32  & 32  & --- & { entropy} \\
\midrule

\multirow{2}{*}{GSM-Symbolic}
  & \llada-8B-Instruct       & 128 & 64  & 32  & { low\_confidence} \\
  & \dream-v0-Instruct-7B      & 128 & 64  & --- & { entropy} \\
\midrule

\multirow{2}{*}{Spider}
  & \llada-8B-Instruct       & 128 & 64  & 32  & { low\_confidence} \\
  & \dream-v0-Instruct-7B      & 128 & 64  & --- & { entropy } \\

\bottomrule
\end{tabular}

\label{tab:gen-settings}
\end{table}

\subsection{Prompt Templates} 
We provide representative prompts and expected outputs for each evaluation task in Table~\ref{tab:prompt_examples}.
{\footnotesize
\setlength{\tabcolsep}{6pt}
\renewcommand{\arraystretch}{1.2}

\begin{longtable}{
@{}
>{\raggedright\arraybackslash}p{2cm}
>{\columncolor{gray!10}\raggedright\arraybackslash}p{8.5cm}
>{\raggedright\arraybackslash}p{2.6cm}
@{}}
\caption{Representative prompts and expected outputs for each evaluation task. Repeated system prompts and in-context examples are omitted for space.}
\label{tab:prompt_examples}\\
\toprule
\textbf{Task} & \textbf{Example Prompt} & \textbf{Example Answer} \\
\midrule
\endfirsthead

\multicolumn{3}{l}{\footnotesize\itshape Table \ref{tab:prompt_examples} continued from previous page}\\
\toprule
\textbf{Task} & \textbf{Example Prompt} & \textbf{Example Answer} \\
\midrule
\endhead

\midrule
\multicolumn{3}{r}{\footnotesize\itshape continued on next page}\\
\endfoot

\bottomrule
\endlastfoot

{\textbf{Function Calling  (BFCL/ xLAM)} \emph{Python Format}} 
&
\emph{System}: You are an expert in composing functions. You are given a question and a set of possible functions. Based on the question, you will need to make one or more function/tool calls to achieve the purpose.\par\vspace{4pt}
If none of the functions can be used, point it out. If the given question lacks the parameters required by the function, also point it out. You should only return the function calls in your response.\par\vspace{4pt}
If you decide to invoke any of the function(s), you MUST put it in the format of:\par\vspace{4pt}
{\ttfamily
\hspace*{1em}[func\_name1(params\_name1=params\_value1,\newline
\hspace*{2em}params\_name2=params\_value2, ...),\newline
\hspace*{1.5em}func\_name2(params)]\par}
\vspace{4pt}
You SHOULD NOT include any other text in the response.\par\vspace{4pt}
At each turn, you should try your best to complete the tasks requested by the user within the current turn. Continue to output functions to call until you have fulfilled the user's request to the best of your ability. Once you have no more functions to call, the system will consider the current turn complete and proceed to the next turn or task.\par\vspace{4pt}
Here is a list of functions in JSON format that you can invoke:\par\vspace{4pt}
{\ttfamily
\hspace*{1em}\{\newline
\hspace*{2em}"name": "calculate\_triangle\_area",\newline
\hspace*{2em}"description": "Calculate the area\newline
\hspace*{3em}of a triangle.",\newline
\hspace*{2em}"parameters": \{\newline
\hspace*{3em}"base": integer,\newline
\hspace*{3em}"height": integer\newline
\hspace*{2em}\}\newline
\hspace*{1em}\}\par}
\vspace{4pt}
\textit{User:} 
Find the area of a triangle with a base of 10 units and height of 5 units.
&
{\ttfamily [calculate\newline\_triangle\_area\newline(base=10, height=5)]}
\\
\midrule
\textbf{Function Calling (BFCL/ xLAM)} \emph{JSON Format} &
\emph{System:} \textit{(...same as above...)}\par\vspace{4pt}
If you decide to invoke any of the function(s), you MUST put it in the format of a JSON array:\par\vspace{4pt}
{\ttfamily
\hspace*{1em}[\{"name": "func\_name1",\newline
\hspace*{2em}"arguments": \{"params\_name1":\newline
\hspace*{3em}params\_value1, ...\}\}, ...]\par}
\vspace{4pt}
\textit{(...same as above...)}
\textit{User:} 
Find the area of a triangle with a base of 10 units and height of 5 units.
&
{\ttfamily [\{"name":\newline "calculate\newline\_triangle\_area",\newline "arguments": \{"base": 10,\newline "height": 5\}\}]}
\\
\midrule
\textbf{Sudoku} &
Fill the positions where the values are 0 in a 4$\times$4 grid with digits 1--4 so that each column, each row, and each of the four 2$\times$2 subgrids that compose the grid contains all of the digits from 1 to 4.\par\vspace{4pt}
\textit{(In-context examples are omitted)}\par\vspace{4pt}
{Input:}\par\vspace{2pt}
{\ttfamily 1003\newline 0300\newline 3210\newline 4132}\par\vspace{4pt}
{Output:}
&
{\ttfamily 1423\newline 2341\newline 3214\newline 4132}
\\
\midrule
\textbf{Countdown} &
Given 4 numbers, use $+,-,\times,\div$ to operate over the first three numbers to achieve the last one.\par\vspace{4pt}
\textit{(In-context examples are omitted)}\par\vspace{4pt}
\textit{Input:} {\ttfamily 61, 83, 8, 30}\par\vspace{4pt}
\textit{Output:}
&
{\ttfamily 83-61=22,\newline 8+22=30}
\\

\midrule
\textbf{Math Reasoning (GSM-Symbolic)} &
\emph{System:} You are an expert in solving grade school math tasks. You will be presented with a grade-school math word problem with symbolic variables and be asked to solve it.\par\vspace{4pt}
Before answering you should reason about the problem (using the {\ttfamily <reasoning>} field in the response described below). Intermediate symbolic expressions generated during reasoning should be wrapped in {\ttfamily <<\,>>}.\par\vspace{4pt}
Only output the symbolic expression wrapped in {\ttfamily <<\,>>} that answers the question. The expression must use numbers as well as the variables defined in the question. You are only allowed to use the following operations: {\ttfamily +, -, /, //, \%, *, **}.\par\vspace{4pt}
You will always respond in the format described below:\par\vspace{4pt}
{\ttfamily Let's think step by step. <reasoning>\newline The final answer is <<symbolic expression>>}\par\vspace{4pt}
\textit{(4 in-context examples are omitted)}\par\vspace{4pt}
\textit{User:} Rania saw a 210-foot whale with 7 72-inch remoras attached to it. What percentage of the whale's body length is the combined length of the remoras?
&
Let's think step by step. The combined length of the remoras in inches is\newline {\ttfamily <<7 * 72>>}. \newline Convert to feet:\newline {\ttfamily <<7 * 72 / 12>>}. Percentage:\newline {\ttfamily <<7 * 72 / 12 / 210 * 100>>}.\newline The final answer is\newline {\ttfamily <<7 * 72 / 12 / 210 * 100>>}.
\\

\midrule
\textbf{Text-to-SQL (Spider)} &
\emph{System:} You are an expert in SQL. Given a database schema and a question, write a correct SQL query.\par\vspace{4pt}
\textit{User:} Here is the database schema:\par\vspace{4pt}
{\ttfamily
CREATE TABLE stadium (\newline
\hspace*{1em}stadium\_id INT PRIMARY KEY,\newline
\hspace*{1em}location TEXT, name TEXT,\newline
\hspace*{1em}capacity INT, highest INT,\newline
\hspace*{1em}lowest INT, average INT\newline
);\par}
\vspace{4pt}
\textit{(...more database omitted...)}\par\vspace{4pt}
Question: How many singers do we have?\par\vspace{4pt}
&
{\ttfamily SELECT count(*)\newline FROM singer}
\\

\end{longtable}

}

\newpage
\subsection{Finite Automata Statistics}
Table~\ref{tab:automaton-sizes}
reports the per-benchmark distribution of automaton sizes for the
\dream-7B and \llada-8B tokenizers. For each benchmark we list
the number of automata $N$ used during evaluation and, across those
$N$ automata, the median, 95th percentile, and maximum of the node and
edge counts. Reporting the median (rather than the mean) reflects the
size of a typical instance, and the 95th percentile and maximum quantify
the heavy right tail driven by a small number of large schemas.
\begin{table}[h]
\centering
\caption{Automaton sizes across benchmarks (\dream-7B and \llada-8B tokenizers). All tasks use DFAs except Spider (NFA due to exponential blowup on DFA). We report number of automata $N$, state counts (tokenizer-independent), and edge counts for each tokenizer (median / P95 / max). GSM-Symbolic uses a single automaton.}
\label{tab:automaton-sizes}
\small
\setlength{\tabcolsep}{4pt}
\renewcommand{\arraystretch}{1.2}
\begin{tabular}{l r r r r | r r r | r r r}
\toprule
& & \multicolumn{3}{c}{\textbf{\# States}} 
& \multicolumn{3}{c}{\textbf{Edges (\dream)}} 
& \multicolumn{3}{c}{\textbf{Edges (\llada)}} \\
\cmidrule(lr){3-5} \cmidrule(lr){6-8} \cmidrule(lr){9-11}
\textbf{Task} & $N$ & Med & P95 & Max 
& Med & P95 & Max 
& Med & P95 & Max \\
\midrule
xLAM \emph{JSON}          & 1000 & 154 & 387 & 1{,}400 & 542 & 1{,}539 & 7{,}414 & 507 & 1{,}415 & 6{,}550 \\
xLAM \emph{Python}        & 1000 & 102 & 287 & 1{,}227 & 397 & 1{,}280 & 6{,}986 & 374 & 1{,}222 & 6{,}154 \\
BFCL NL \emph{JSON}       & 1000 & 122 & 357 & 818     & 409 & 1{,}213 & 3{,}968 & 386 & 1{,}134 & 3{,}572 \\
BFCL L \emph{JSON}        & 1351 & 385 & 997 & 2{,}459 & 1{,}317 & 3{,}932 & 9{,}901 & 1{,}233 & 3{,}672 & 9{,}128 \\
BFCL NL \emph{Python}     & 1000 &  90 & 272 & 713     & 319 &   975 & 3{,}738 & 304 &   936 & 3{,}342 \\
BFCL L \emph{Python}      & 1351 & 301 & 778 & 1{,}613 & 1{,}112 & 2{,}839 & 7{,}576 & 1{,}042 & 2{,}694 & 6{,}704 \\
Sudoku                    &  900 &  21 &  21 &   21    &  21 &    21 &   21    &  21 &    21 &   21 \\
Countdown                 &  993 &  47 &  55 &   77    &  53 &    63 &   88    &  52 &    62 &   87 \\
GSM-Symbolic              &    1 &  56 &  -- &   56    & 745 &    -- &  745    & 748 &    -- &  748 \\
Spider                    &   20 & 8{,}796 & 15{,}573 & 19{,}509 
& 85{,}415 & 168{,}920 & 171{,}580 
& 70{,}477 & 142{,}469 & 144{,}591 \\
\bottomrule
\end{tabular}
\end{table}

\newpage
\section{Qualitative Examples on BFCL}
Table~\ref{tab:qual_violation} shows four representative cases on \texttt{live\_simple}: in each one the model has the right intent but cannot keep its output inside the tool schema. Constraining decoding to the schema FA eliminates these failures by construction, recovering all four examples.
\newcommand{\baseout}[1]{\textcolor{red!75!black}{\texttt{#1}}}
\newcommand{\oursout}[1]{\textcolor{green!50!black}{\texttt{#1}}}

\begin{table}[h]
\caption{Handling constraint violations on BFCL-Live (\dream-Instruct, $T{=}0$).
The baseline often produces outputs that violate the tool specification, and receives $0\%$ accuracy. E.g., prose instead of a tool call, incorrect casing, or hallucinated function names. Our constrained generation guarantees syntactically valid tool calls by construction (constraint satisfaction rate $=100\%$) and is also semantically correct on these examples.
}
\centering
\small
\setlength{\tabcolsep}{4pt}
\renewcommand{\arraystretch}{1.15}
\begin{tabular}{@{}>{\bfseries}l p{0.84\linewidth}@{}}
\toprule
\multicolumn{2}{@{}l}{\textit{(a) JSON: baseline emits prose instead of a tool call.}} \\
\midrule
User & \emph{``I have a customer query asking for `Need ATM location'. Which category should it be classified under?''} \\
Tool & \texttt{record(atm\_finder\_start: list[str], \ldots)} -- $6$ enum-keyed list args. \\
Base & \baseout{The customer query 'Need ATM location' should fall under the `atm\_finder\_start' category.} \quad \textcolor{red!75!black}{\scriptsize($\hookrightarrow$ Invalid format: not a JSON tool-call list.)} \\
Ours & \oursout{[\{"name": "record", "arguments": \{"atm\_finder\_start": ["Need ATM location"]\}\}]}  \quad \textbf{(correct)} \\
\midrule
\multicolumn{2}{@{}l}{\textit{(b) JSON: wrong argument-name casing breaks the schema.}} \\
\midrule
User & \emph{``Retrieve today's alerts for sensor Q3CC-CRT3-SZ2G, max $10$ alerts per page.''} \\
Tool & \texttt{get\_sensor\_alerts(perPage: int*, sensorSerial: str*, \ldots)} \\
Base & \baseout{[\{"name": "get\_sensor\_alerts", "arguments": \{\textbf{"PerPage"}: 10, "t0": "2023-10-01", "sensorSerial": "Q3CC-CRT3-SZ2G"\}\}]} \quad \textcolor{red!75!black}{\scriptsize($\hookrightarrow$ missing required `perPage'; \texttt{t0} hallucinated.)} \\
Ours & \oursout{[\{"name": "get\_sensor\_alerts", "arguments": \{"perPage": 10, "sensorSerial": "Q3CC-CRT3-SZ2G"\}\}]} \quad \textbf{(correct)} \\
\midrule
\multicolumn{2}{@{}l}{\textit{(c) Python: hallucinated function name (dotted-method that doesn't exist in the tool list).}} \\
\midrule
User & \emph{``I've completed the task `Submit monthly financial report'. Mark it as completed.''} \\
Tool & \texttt{todo(type: enum["complete", \ldots], content: str)} \\
Base & \baseout{[todo.complete(content="Submit monthly financial report")]} \quad \textcolor{red!75!black}{\scriptsize($\hookrightarrow$ function name `todo.complete' not in tool list.)} \\
Ours & \oursout{[todo(type="complete", content="Submit monthly financial report")]} \quad \textbf{(correct)} \\
\midrule
\multicolumn{2}{@{}l}{\textit{(d) Python: typo in argument name.}} \\
\midrule
User & \emph{``Find me a Comic movie directed by Nitesh Tiwari.''} \\
Tool & \texttt{Movies\_3\_FindMovies(directed\_by, genre, cast)} \\
Base & \baseout{[Movies\_3\_FindMovies(\textbf{ddirected\_by}="Nitesh Tiwari", genre="Comedy")]} \quad \textcolor{red!75!black}{\scriptsize($\hookrightarrow$ unexpected parameter `ddirected\_by'.)} \\
Ours & \oursout{[Movies\_3\_FindMovies(directed\_by="Nitesh Tiwari", genre="Comedy")]} \quad \textbf{(correct)} \\
\bottomrule
\end{tabular}

\label{tab:qual_violation}
\end{table}

\newpage
Constraining the output to the schema is not the same as constraining it to the gold answer. Table~\ref{tab:qual_residual} shows three failures that survive our decoder: the schema accepts every output, but the value the model chose disagrees with the reference. These are the errors a token-level grammar cannot fix.
\begin{table}[h]

\centering
\caption{Failure modes under constrained decoding on BFCL-Live (\dream-Instruct, $T{=}0$).
All outputs satisfy the schema FA (valid format, function name, and argument types), but are still rejected by the official scorer because some argument values differ from the ground truth. These errors cannot be ruled out by token-level constraints: the model must choose among multiple schema-valid options and selects the wrong one.}
\small
\setlength{\tabcolsep}{4pt}
\renewcommand{\arraystretch}{1.15}
\begin{tabular}{@{}>{\bfseries}l p{0.84\linewidth}@{}}
\toprule
\multicolumn{2}{@{}l}{\textit{(a) Over-specified \emph{location} -- model copies the user's parenthetical landmark.}} \\
\midrule
User & \emph{``Tell me the current weather in Yosemite National Park, which locates at Mariposa, CA, in Celsius.''} \\
Tool & \texttt{get\_current\_weather(location: str, unit: str)} \\
Ours & \oursout{[\{"name": "get\_current\_weather", "arguments": \{"location": "Yosemite National Park, Mariposa, CA", "unit": "celsius"\}\}]} \\
Gold & \texttt{location = "Mariposa, CA"} \quad \textcolor{gray}{\scriptsize($\hookrightarrow$ schema OK; string-mismatch for \texttt{location}.)} \\
\midrule
\multicolumn{2}{@{}l}{\textit{(b) Hallucinated argument value -- enum-legal but not the gold default.}} \\
\midrule
User & \emph{``Set the air conditioner to cool mode.''} (no wind speed, no temperature given) \\
Tool & \texttt{ThinQ\_Connect(body: dict)} with \texttt{body.windStrength} $\in$ \{LOW, MID, HIGH\}. \\
Ours & \oursout{[ThinQ\_Connect(body=\{\ldots,\,"windStrength": "HIGH",\,"coolTargetTemperature": 24,\,\ldots\})]} \\
Gold & \texttt{windStrength = "MID"}, no \texttt{targetTemperature}. \quad \textcolor{gray}{\scriptsize($\hookrightarrow$ FA accepts every enum value; LM picked the wrong one.)} \\
\midrule
\multicolumn{2}{@{}l}{\textit{(c) Missing optional-but-inferable argument.}} \\
\midrule
User & \emph{``I made a \$59.99 purchase in San Francisco, California. How much sales tax?''} \\
Tool & \texttt{calculate\_tax(purchase\_amount, state, county, city, \ldots)} \\
Ours & \oursout{[\{"name": "calculate\_tax", "arguments": \{"purchase\_amount": 59.99, "state": "CA"\}\}]} \\
Gold & \texttt{state="CA", county="San Francisco", city="San Francisco"} \quad \textcolor{gray}{\scriptsize($\hookrightarrow$ \texttt{county} and \texttt{city} are schema-optional; LM omitted them.)} \\
\bottomrule
\end{tabular}

\label{tab:qual_residual}
\end{table}

\section{Additional Experiment Results}
\textbf{Scaling with the difficulty of problems.} 
In Sudoku, we vary the number of prefilled cells from 12 to 4, with fewer prefilled cells corresponding to more difficult puzzles. As shown in Figure~\ref{fig:sudoku-curves}, baseline performance drops sharply as the puzzles become harder. A common failure mode is that the model modifies prefilled cells, violating the puzzle constraints. Enforcing that these cells remain fixed effectively reduces the task to a text-infilling problem, leading to substantial improvements. Confidence-based scheduling further boosts performance by guiding decoding toward more likely valid solutions, particularly on the hardest puzzles where the base model struggles.
\begin{figure}[h]
  \centering
  \begin{minipage}{0.51\linewidth}
    \centering
    \includegraphics[width=0.524\linewidth]{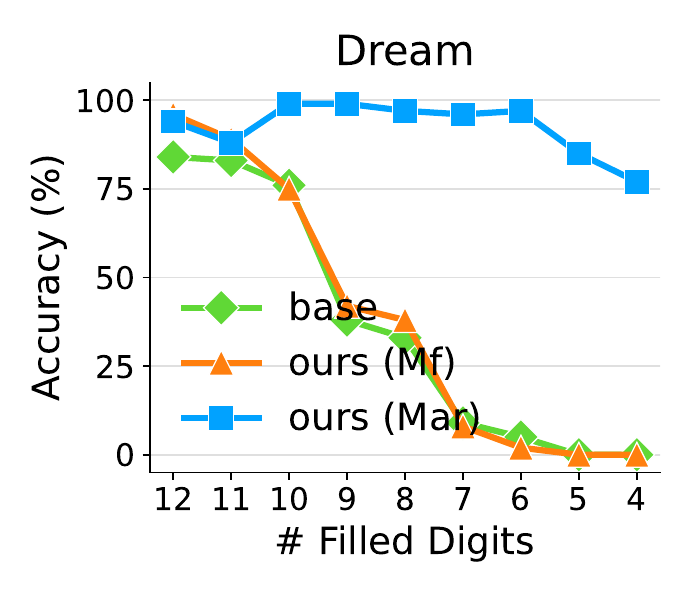}\hfill
    \includegraphics[width=0.472\linewidth]{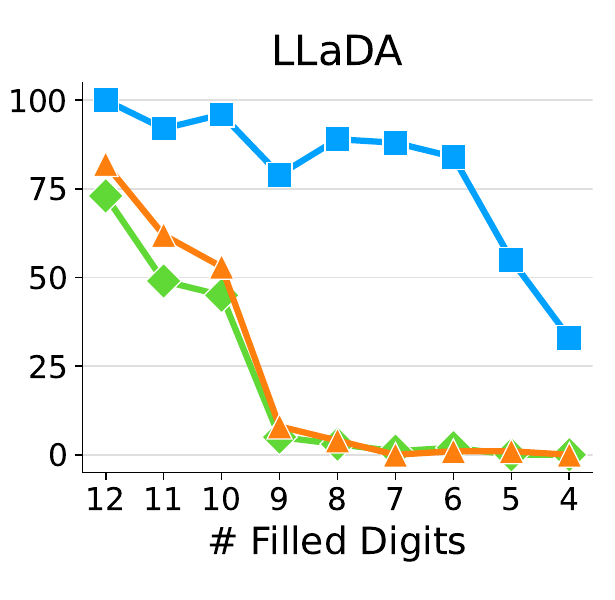}
  \end{minipage}
  \caption{Sudoku 4$\times$4 accuracy versus the number of prefilled digits (fewer digits indicate harder puzzles). {Mf} uses mean-field confidence from the base model, while {Mar} uses marginals under the constrained distribution.}
  \label{fig:sudoku-curves}
\end{figure}

\end{document}